\documentclass{article}

\PassOptionsToPackage{svgnames,dvipsnames}{xcolor}

\usepackage[accepted]{icml2024_template/icml2024_cr_version}

\usepackage{amsfonts}
\usepackage{nicefrac}       %
\usepackage{xfrac}
\usepackage{amsmath}
\usepackage{amsthm}
\usepackage{amssymb}
\usepackage{cancel}
\usepackage{mathtools}
\usepackage{bbm}
\usepackage{bm}

\usepackage{cases}         %

\usepackage{mathalpha}
\usepackage{xspace} %

\usepackage{microtype}      %
\usepackage{enumitem}

\usepackage{pifont}

\newcommand{\reals}{\mathbb{R}}

\newcommand{\defas}{\triangleq}

\newcommand{\lpar}{\left(}
\newcommand{\rpar}{\right)}

\newcommand{\lspar}{\left[}
\newcommand{\rspar}{\right]}

\newcommand{\T}{^{\top}} %
\newcommand{\bg}[1]{\boldsymbol{#1}}

\newcommand{\eye}{\mathbf{I}}
\newcommand{\vzero}{\boldsymbol{0}}
\newcommand{\vb}{\boldsymbol{b}}
\newcommand{\vc}{\boldsymbol{c}}
\newcommand{\ve}{\boldsymbol{e}}
\newcommand{\vf}{\boldsymbol{f}}
\newcommand{\vg}{\boldsymbol{g}}
\newcommand{\vh}{\boldsymbol{h}}

\newcommand{\vw}{\boldsymbol{w}}
\newcommand{\vx}{\boldsymbol{x}}

\newcommand{\vz}{\boldsymbol{z}}

\newcommand{\vtheta}{\boldsymbol{\theta}}
\newcommand{\vphi}{\boldsymbol{\phi}}

\newcommand{\vxi}{\boldsymbol{\xi}}

\newcommand{\mH}{\boldsymbol{H}}
\newcommand{\mA}{\boldsymbol{A}}
\newcommand{\mU}{\boldsymbol{U}}

\newcommand{\mPhi}{\boldsymbol{\Phi}}

\newcommand{\mXi}{\boldsymbol{\Xi}}

\newcommand{\Lag}{\mathfrak{L}}
\newcommand{\vlambda}{\boldsymbol{\lambda}}
\newcommand{\vmu}{\boldsymbol{\mu}}

\newcommand{\vcx}{\vc(\vx)}
\newcommand{\vgx}{\vg(\vx)}
\newcommand{\vhx}{\vh(\vx)}

\newcommand{\lrp}{\eta_{\text{primal}}}
\newcommand{\lrd}{\eta_{\text{dual}}}

\newcommand{\Proj}{\Pi}

\newcommand{\X}{\mathcal{X}}

\newcommand{\blobletter}[1]{\raisebox{.5pt}{\textcircled{\raisebox{-.8pt}{{\hspace{-1mm} \small #1}}}}}

\newcommand{\Jac}{\mathcal{J}}

\newcommand{\algo}[1]{{\small {\textsc{#1}}}}
\newcommand{\Polyak}{\algo{Polyak}}
\newcommand{\Nesterov}{\algo{Nesterov}}
\newcommand{\GA}{\algo{GA}}
\newcommand{\Adam}{\algo{Adam}}

\newcommand{\nuPI}{$\bg{\nu}$\algo{PI}\xspace}
\newcommand{\kp}{\kappa_p}
\newcommand{\ki}{\kappa_i}
\newcommand{\kd}{\kappa_d}

\definecolor{lightgray}{RGB}{230, 230, 230}
\definecolor{mathred}{RGB}{204, 69, 90}
\definecolor{mathblue}{RGB}{4, 78, 112}
\definecolor{mathgreen}{RGB}{1, 135, 70}

\DeclareMathOperator{\argminin}{argmin}  
\newcommand{\argmin}[1]{\underset{#1}{\argminin} \;}

\definecolor{linkcolor}{RGB}{0,120,130}

\usepackage[colorlinks=true,
    linkcolor=linkcolor,
    citecolor=linkcolor,
    filecolor=linkcolor,
    urlcolor=linkcolor,
    pagebackref=true]{hyperref} 
\usepackage{url}

\renewcommand*\backref[1]{\ifx#1\relax \else (Cit. on p. #1) \fi}

\usepackage[capitalize]{cleveref}

\Crefname{algorithm}{Algo.}{Algos.}
\Crefname{theorem}{Thm.}{Thms.}
\Crefname{lemma}{Lemma}{Lems.}
\Crefname{appendix}{Appx.}{Appx.}

\theoremstyle{plain}
\newtheorem{theorem}{Theorem}
\newtheorem*{nonumtheorem}{Theorem}

\newtheorem{lemma}[theorem]{Lemma}

\theoremstyle{remark}

\theoremstyle{problem}

\usepackage[textsize=tiny]{todonotes}
\usepackage{mathtools}
\usepackage{multirow}
\usepackage{booktabs}
\usepackage{adjustbox}

\usepackage{caption}
\usepackage{subcaption}

\usepackage{tikz}
\usepackage{pgfplots}
\pgfplotsset{compat=1.17}
\usepackage{standalone}
\usetikzlibrary{shapes,arrows,fit,positioning, calc, decorations.markings}
\definecolor{captiongray}{RGB}{100,100,100}
\newcommand{\captioncomment}[1]{{\color{captiongray} \footnotesize #1}}

\usepackage[toc,page,header]{appendix}
\usepackage[nohints]{minitoc}

\newcommand{\graytext}[1]{{\color{gray}#1}}

\definecolor{cerulean}{rgb}{0.0, 0.48, 0.65}
\definecolor{mygreen}{RGB}{129, 199, 148}
\definecolor{lightred}{RGB}{247, 171, 171}

\usepackage[framemethod=TikZ]{mdframed}

\newcounter{algrthm}
\renewcommand{\thealgrthm}{\arabic{algrthm}}
\Crefname{algrthm}{Algorithm}{Algorithms}

\newcounter{thm}
\renewcommand{\thethm}{\arabic{thm}}
\Crefname{thm}{Theorem}{Theorems}

\icmltitlerunning{On PI Controllers for Updating Lagrange Multipliers in Constrained Optimization}

\begin{document}

\twocolumn[

\vspace{-2ex}
\icmltitle{On PI Controllers for Updating\\Lagrange Multipliers in Constrained Optimization}

\icmlsetsymbol{equal}{*}
\icmlsetsymbol{MilaUdeM}{$\dagger$}
\icmlsetsymbol{CIFAR}{$\ddagger$}

\begin{icmlauthorlist}
\icmlauthor{Motahareh Sohrabi}{equal,MilaUdeM}
\icmlauthor{Juan Ramirez}{equal,MilaUdeM}
\icmlauthor{Tianyue H. Zhang}{MilaUdeM}
\icmlauthor{Simon Lacoste-Julien}{MilaUdeM,CIFAR}
\icmlauthor{Jose Gallego-Posada}{MilaUdeM}

\end{icmlauthorlist}

\icmlcorrespondingauthor{Juan Ramirez}{\texttt{juan.ramirez@mila.quebec}}

\icmlkeywords{Machine Learning, ICML, Lagrangian Constrained Optimization, Lagrangian games, PI Controller, damping, sparsity}

\vskip 0.3in
]

\printAffiliationsAndNotice{* Equal contribution. $\dagger$ Mila---Quebec AI Institute and DIRO, Université de Montréal. $\ddagger$ Canada CIFAR AI Chair} %

\begin{abstract}

Constrained optimization offers a powerful framework to prescribe desired behaviors in neural network models. 
Typically, constrained problems are solved via their min-max Lagrangian formulations, which exhibit unstable oscillatory dynamics when optimized using gradient descent-ascent. 
The adoption of constrained optimization techniques in the machine learning community is currently limited by the lack of reliable, general-purpose update schemes for the Lagrange multipliers.
This paper proposes the \nuPI algorithm and contributes an optimization perspective on Lagrange multiplier updates based on PI controllers, extending the work of \citet{stooke2020responsive}.
We provide theoretical and empirical insights explaining the inability of momentum methods to address the shortcomings of gradient descent-ascent, and contrast this with the empirical success of our proposed \nuPI controller. Moreover, we prove that \nuPI generalizes popular momentum methods for single-objective minimization.
Our experiments demonstrate that \nuPI reliably stabilizes the multiplier dynamics and its hyperparameters enjoy robust and predictable behavior.

\end{abstract}

\doparttoc %
\faketableofcontents %
\part{} %

\vspace{-10ex}

\section{Introduction}
\label{sec:introduction}

The need to enforce complex behaviors in neural network models has reinvigorated the interest of the machine learning community in constrained optimization techniques. Recent applications include fairness \cite{cotter2019proxy, zafar2019fairness, fioretto2020LagrangianDuality, hashemizadeh2023balancing}, sparsity \cite{gallego2022controlled}, active learning \cite{elenter2022lagrangian}, reinforcement learning \cite{stooke2020responsive,farahmand2021pid} and model quantization \cite{hounie2023neural}. 

\begin{figure}[h]
    \vspace{-2ex}
    \centering
    \hspace*{-2ex} \includegraphics[scale=0.85]{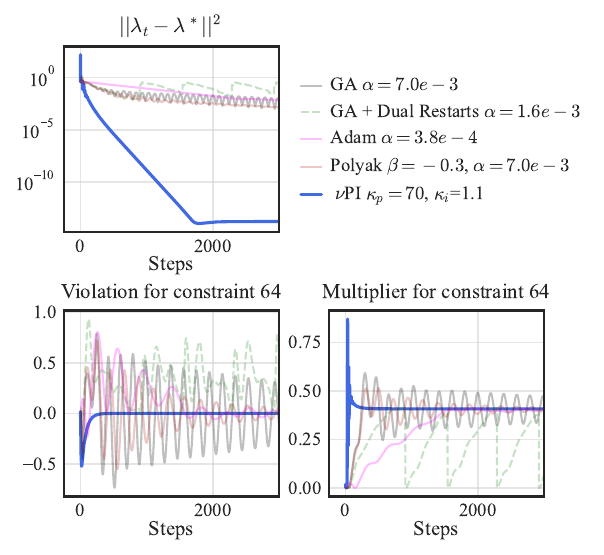}
    \vspace{-5ex}
    \caption[Dynamics for different dual optimizers on a hard-margin SVM problem (\cref{eq:svm_problem}).]{Dynamics for different dual optimizers on a hard-margin SVM problem (\cref{eq:svm_problem}). Amongst the tested methods, \textbf{\nuPI is the only method to successfully converge to the optimal dual variables}.
    \captioncomment{Each optimizer uses the best hyperparameters found after a grid-search aiming to minimize the distance to the optimal $\vlambda^*$ after 5.000 steps. For improved readability, the plot shows the first 3.000 steps.
    Constraint 64 corresponds to a support vector. All methods achieved perfect training accuracy.
    }
    }
    \label{fig:svm_dynamics_plot}
    \vspace*{-5ex}
\end{figure}

Algorithmic approaches based on the Lagrangian min-max representation of the original constrained optimization problem \citep[\S 5]{boyd2004convex} are commonly preferred in the context of neural networks since (i) they are amenable to inexact, gradient-based optimization \citep[\S 5.2]{bertsekas2016nonlinear}, (ii) making it easy to incorporate constraints into existing pipelines for unconstrained optimization \cite{cotter2019proxy, gallegoPosada2022cooper}, and (iii) they do not require special structure in the objective or constraint functions (such as convexity or efficient projection onto the feasible set \cite{nocedal2006numerical}).

Despite their wider applicability, solving Lagrangian problems involving neural networks is challenging as it simultaneously entails the difficulties of nonconvex optimization on large-scale models \cite{bottou2018optimization}, and the potential for instability and oscillations due to the adversarial min-max nature of the Lagrangian \cite{stooke2020responsive}.

Lagrangian problems are commonly optimized using some variant of gradient-descent ascent (GDA) \cite{arrow1958studies}. 
Despite local convergence results in idealized settings \cite{lin2020gradient, zhang2022near}, the optimization dynamics of GDA typically exhibit instabilities, overshoot or oscillations \citep{platt1988constrained,gidel2019negative,stooke2020responsive,gallego2022controlled}.

Alleviating the shortcomings of GDA on Lagrangian problems is an important step towards widespread adoption constrained optimization in deep learning.
Recently, \citet{stooke2020responsive} proposed a solution based on a PID controller \cite{astrom1995pid} for updating the Lagrange multipliers in safety-constrained reinforcement learning problems. Our manuscript expands on their work by providing an optimization-oriented analysis of \nuPI (\cref{algo:nupi}), a related PI controller that incorporates an exponential moving average on the error signal.

\cref{fig:svm_dynamics_plot} illustrates how our proposed \nuPI controller successfully dampens the oscillations on a hard-margin SVM task, achieving fast convergence to the optimal Lagrange multipliers. In contrast, a wide range of popular methods for single-objective minimization exhibit unstable, oscillatory dynamics and fail to converge in this task. See \S\ref{sec:svm} for further details on this experiement.

\textbf{Contributions:} \blobletter{1}
We introduce the \nuPI algorithm (\S \ref{sec:nupi}) and prove that \nuPI generalizes popular momentum methods like \Polyak~and \algo{Nesterov} (\cref{thm:um_as_nupi}), as well as traditional PI controllers. 
\blobletter{2} We provide conceptual insights explaining how \nuPI improves the dynamics of the Lagrange multipliers: \S\ref{sec:one_over_x} presents a qualitative analysis of the updates executed by the \nuPI algorithm in contrast to gradient ascent; in \S\ref{sec:oscillator} we study the spectral properties of the continuous-time system.
\blobletter{3} In \S\ref{sec:nupi_practical}, we provide a heuristic to tune the new hyperparameter $\kp$ of the \nuPI algorithm; we also demonstrate that it has a monotonic effect in the damping of oscillations.
\blobletter{4} Our experiments on hard-margin SVMs, sparsity tasks using ResNets, and algorithmic fairness demonstrate that \nuPI leads to improved stability and convergence.

\textbf{Code:}
\href{https://github.com/motahareh-sohrabi/nuPI}{\texttt{github.com/motahareh-sohrabi/nuPI}}

\textbf{Scope:} Due to the highly specialized techniques used for 
training neural networks \citep{dahl2023benchmarking}, in this work we concentrate on iterative schemes that do \textit{not} modify the optimization protocol used on the model parameters. In other words, we restrict our attention to update schemes on the Lagrange multipliers only, which allows us to reuse the same optimizer choices for the (primal) model parameters as used in the unconstrained setting.

\vspace{-1ex}

\section{Related Works}
\label{sec:related_work}

\textbf{Constrained optimization.} We are interested in Lagrangian methods \citep{arrow1958studies} that allow tackling general (nonconvex) constrained optimization problems with differentiable objective and constraints. Classical constrained optimization \cite{nocedal2006numerical, bertsekas2016nonlinear} techniques include projection methods \citep{bertsekas1976goldstein}, barrier methods \citep{dikin1967iterative}, and methods of feasible directions \citep{frank1956algorithm,zoutendijk1960methods}. 
These approaches usually make assumptions on the structure of the problem, such as convexity of the objective or constraints, the existence of an efficient projection operator onto the feasible set, or access to a linear minimization oracle. Such assumptions restrict their applicability to deep learning tasks.  
Other popular techniques such as penalty methods \citep{nocedal2006numerical} and the method of multipliers \citep{bertsekas1975method}, apply to general nonconvex problems, but are outside the scope of this work.  

\textbf{Min-max optimization.} The Lagrangian formulation of a nonconvex constrained optimization problem leads to a nonconvex concave min-max problem. 
Under idealized assumptions, gradient descent-ascent has local convergence guarantees for said problems \citep{lin2020gradient}, but may exhibit oscillations \citep{platt1988constrained,gidel2018variational}. 
Under stronger assumptions, extragradient \citep{korpelevich1976extragradient} and the optimistic gradient method \citep{popov1980modification} converge at a nearly optimal rate \citep{mokhtari2020convergence}. These methods, as well as \Polyak~with negative momentum \citep{gidel2019negative} and PID controllers \citep{stooke2020responsive}, have been shown to dampen the oscillations of GDA. 
However, negative momentum may be suboptimal for strongly convex-strongly concave min-max problems \citep{zhang2021suboptimality}. 

Our work focuses on the \textit{dynamics} of Lagrangian games. We provide insights on why popular techniques for minimization may exacerbate oscillations and overshoot, and why PI controllers can be effective at damping oscillations. Our proposed method \nuPI is a generalization of both (negative) momentum and the optimistic gradient method.

\textbf{PID controllers and optimization.} 
\citet{an2018pid} studied PID control for training machine learning models by considering the negative loss gradient as the error signal to the controller.
PID controllers have been shown to generalize gradient descent \citep{hu2017control} and momentum \citep{recht2018blog}.
\citet{stooke2020responsive,casti2023control} have highlighted the effectiveness of controllers at optimizing constrained optimization tasks. 

In this work, we propose a PI-like update rule for the dual variables in a Lagrangian min-max game. We prove our algorithm generalizes momentum methods and we provide conceptual insights to support the empirical effectiveness of PI controllers in reducing oscillations and overshoot in the constrained optimization dynamics. In \cref{appx:pid_related_works}, we elaborate on the distinctions between our work and existing research on PID controllers for optimization.

\section{Lagrangian Optimization}
\label{sec:lagrangian_optimization}

Consider a constrained optimization problem with $m$ inequality and $n$ equality constraints, represented by functions $\vg: \X \rightarrow \reals^m $ and $\vh:~\X~\rightarrow ~\reals^n$, respectively:
\begin{align}
    \label{eq:cmp_definition}
    \underset{\vx}{\text{min}} \, f(\vx) \hspace{2mm}
    \text{subject to} \hspace{2mm}  \vgx \le \vzero \hspace{2mm}
    \text{and} \hspace{2mm} \vhx = \vzero.
\end{align}
We do not make any assumptions on the functions $f$, $\vg$, and $\vh$ beyond almost-everywhere differentiability. We refer to the values of $\vg$ and $\vh$ as the \textit{constraint violations}. In particular, we are interested in optimization problems where $\vx$ corresponds to the parameters of a neural network, leading to objective and constraint functions that may be nonconvex. This typically precludes the use of ``classical'' constrained optimization methods, as those discussed in \S\ref{sec:related_work}.

The Lagrangian min-max problem associated with the constrained optimization problem in \cref{eq:cmp_definition} is given by:
\begin{equation}
    \label{eq:lagrangian_x_lambda_mu}
    \underset{\vx}{\text{min}} \underset{\vlambda \ge \vzero, \, \vmu}{\text{max}}  \Lag(\vx, \vlambda, \vmu) \triangleq f(\vx) + \vlambda\T \vgx + \vmu\T \vhx,
\end{equation}
where $\vlambda$ and $\vmu$ are vectors of \textit{Lagrange multipliers} associated with the inequality and equality constraints, respectively. \cref{eq:lagrangian_x_lambda_mu} constitutes a 
nonconvex-concave zero-sum game between $\vx$ (known as the \textit{primal} player) and $\{ \vlambda, \vmu \}$ (known as the \textit{dual} player). We are interested in algorithmic approaches that identify saddle points of the Lagrangian $\Lag(\vx, \vlambda, \vmu)$ as these correspond to constrained optima.

In general, Lagrangian-based approaches do not constitute \textit{feasible methods} (i.e. visiting only feasible iterates). We judge a method's success based on its asymptotic feasibility, or at the end of a pre-determined optimization budget. 

\textbf{Simultaneous updates.} The simplest algorithm to solve the problem in \cref{eq:lagrangian_x_lambda_mu} is simultaneous gradient descent-ascent (GDA) \citep{arrow1958studies}:
\begin{align*}
\label{eq:simultaneous_gda_updates}
    & \hspace{4mm} \vmu_{t+1} \leftarrow  \vmu_{t} +  \lrd \,  \nabla_{\vmu} \Lag(\vx_{t}, \vlambda_{t}, \vmu_{t}) = \vmu_t + \lrd \, \vh(\vx_t)  \\
    &\begin{cases}
    \hat{\vlambda}_{t+1} \leftarrow  \vlambda_{t} +  \lrd \,  \nabla_{\vlambda} \Lag(\vx_{t}, \vlambda_{t}, \vmu_{t}) = \vlambda_t + \lrd \, \vg(\vx_t) \\
    \vlambda_{t+1} \leftarrow \Proj_{\reals^m_+}(\hat{\vlambda}_{t+1}) = \max \left( 0, \hat{\vlambda}_{t+1} \right)
    \end{cases} \\
    & \hspace{4mm} \vx_{t+1} \leftarrow \vx_t - \lrp \,  \nabla_{\vx} \Lag(\vx_{t}, \vlambda_{t}, \vmu_{t}),
\end{align*}
where the middle two equations execute a projected gradient-ascent step enforcing the non-negativity of the multipliers~$\vlambda$.

To simplify notation, we will group the dual variables as $\vtheta~=~\lspar \vlambda, \vmu \rspar\T$ and the constraints $\vc(\vx) = \lspar \vgx, \vhx \rspar\T$ which yields the concise Lagrangian problem:
\begin{equation}
    \label{eq:lagrangian}
    \underset{\vx}{\text{min}} \underset{\vtheta \in \reals^m_+ \times \reals^n}{\text{max}}  \Lag(\vx, \vtheta) \triangleq f(\vx) + \vtheta\T \vcx
\end{equation}
Note that the primal update direction $\nabla_{\vx} \Lag$ is a linear combination of the objective and constraint gradients\textemdash which can be efficiently computed using automatic differentiation, without storing $\nabla f$ and $\Jac \vc$\footnote{$\Jac\vf \defas \begin{bmatrix} \nabla f_1 & \cdots & \nabla f_p \end{bmatrix} \in \reals^{d \times p}$ denotes the (transpose) Jacobian matrix of a function $\vf: \reals^d \rightarrow \reals^p$.} separately. 
On the other hand, $\nabla_{\vtheta} \Lag = \vcx$, and thus the GDA update on the multipliers corresponds to the \textit{integration} (i.e. accumulation) of the constraint violations over time. We highlight that the cost of updating the Lagrange multipliers is typically negligible relative to the cost of computing $f$ and $\vc$.

\textbf{Alternating updates.}
Prior work has demonstrated the advantages of alternating updates in min-max optimization: \citet{zhang2022near} established that alternating GDA achieves a near-optimal local convergence rate for strongly concave-strongly convex problems (strictly better than simultaneous GDA); 
\citet{gidel2018variational} showed that alternating GDA leads to bounded iterates on smooth bilinear games, as opposed to divergence for simultaneous updates.
Besides the improved convergence and stability benefits, alternating updates are particularly suitable for Lagrangian games from a computational standpoint due to the linear structure of the Lagrangian with respect to the dual variables.
Concretely, consider the \textit{alternating} update scheme:
\begin{align}
\label{eq:alternating_gda_updates}
    \begin{split}
    &\hspace{-5mm} \begin{cases}
    \hat{\vtheta}_{t+1} \leftarrow \vtheta_{t} +  \lrd \,  \nabla_{\vtheta} \Lag(\vx_{t}, \vtheta_{t}) = \vtheta_t + \lrd \, \vc(\vx_t) \\
    \vtheta_{t+1} \leftarrow 
    \Proj_{\reals^m_+ \times \reals^n}(\hat{\vtheta}_{t+1})
    \end{cases} \\
    &\vx_{t+1} \leftarrow \vx_t - \lrp \,  \nabla_{\vx} \Lag(\vx_{t}, \vtheta_{t+1}) \\
    & \hspace{9mm} = \vx_t - \lrp \lpar \nabla f(\vx_t)  + \Jac \vc(\vx_t) \, \vtheta_{t} \rpar
    \end{split}
\end{align}

The alternating updates in \cref{eq:alternating_gda_updates}, only require computing $f(\vx_t)$ and $\vc(\vx_t)$ once, just as when performing simultaneous updates. 
In a general zero-sum game, where $\Lag(\vx_t, \vtheta_t)$ does not decouple as in the Lagrangian case, the second part of the alternation might require re-evaluating $\Lag(\vx_t, \vtheta_{t+1})$ entirely. However, note that thanks to the affine structure of $\Lag$ with respect to $\vtheta$, the update on $\vx$ can be calculated efficiently without having to re-evaluate $f$ or $\vc$.

These theoretical and practical advantages motivate our decision to concentrate on alternating update schemes like \cref{eq:alternating_gda_updates} for solving the problem in  \cref{eq:lagrangian} in what follows. 

\textbf{Practical remarks.} In practice, updates on the primal variables require more sophisticated methods (with intricate hyperparameter tuning) than the plain gradient descent update presented in \cref{eq:alternating_gda_updates} to achieve good performance, including any number of highly specialized procedures developed for training neural networks \citep{dahl2023benchmarking}. 

Moreover, for certain applications, a training pipeline designed to minimize a single, \textit{unconstrained} objective might be in place. In these cases, it is desirable to develop update schemes for the Lagrange multipliers that allow for seamlessly incorporating constraints into the model development pipeline \textit{without having to engineer from scratch a new recipe for training the model.} 

In this paper, we concentrate on different update schemes for the Lagrange multipliers and assume that a well-tuned optimizer for the model parameters is available. 

\textbf{Shortcomings of gradient ascent.} As mentioned previously, gradient ascent (\GA) on the Lagrange multipliers corresponds to accumulating the observed constraint violations over time. For simplicity, let us concentrate on a single inequality constraint $c(\vx)$. Whenever the constraint is being violated (resp. satisfied), the violation is positive $c(\vx) > 0$ (resp. negative) and thus the value of the corresponding multiplier is increased (resp. decreased) by $\lrd c(\vx)$. Recall that the projection step ensures that the inequality multipliers remain non-negative. 

Therefore, the value of the multiplier depends on the entire optimization trajectory through the value of the observed violations. In particular, after a long period of infeasibility, the value of the multiplier will be large, biasing the gradient $\nabla_{\vx} \Lag$ towards reducing the violation and thus improving the feasibility of the model.

An insufficient increase of the multiplier will cause the constraint to be \textit{ignored}, while an excessively large value of the multiplier will lead the constraint to be enforced \textit{beyond} the prescribed constraint level. The latter behavior can also occur if the multiplier fails to decrease sufficiently fast once the constraint is satisfied. Repeated cycles of insufficient or excessive change in the multiplier manifest in ignoring or overshooting, thus forming oscillations. See \cref{fig:svm_dynamics_plot,fig:dynamics_ga} for illustrations of these behaviors.

\begin{figure}[h!]
    \centering
    \vspace{-1ex}
    \includegraphics[scale=1]{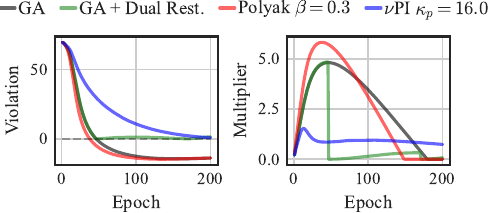}
    \vspace{-3ex}
    \caption[Constraint dynamics for \GA, \Polyak~and \nuPI in a sparsity task (\S\ref{sec:sparsity}).]{Constraint dynamics for \GA, \Polyak~ and \nuPI in a sparsity task (\S\ref{sec:sparsity}). Constrained optimal solutions for this problem lie at the boundary of the feasible set. The excessive growth in the value of the multiplier for GA causes the constraint to overshoot into the interior of the feasible set. \textbf{The improved multiplier updates of the \nuPI algorithm remove the overshoot in the constraint and multiplier.}}
    \label{fig:dynamics_ga}
\end{figure}

\vspace{-3ex}

In short, an ideal update rule for the multiplier would behave \textit{adaptively}, based on the observed violations throughout the execution of the optimization. This begs the question of whether existing adaptive optimization such as \Polyak, \Nesterov~and \Adam~would reliably resolve these issues. Sections \ref{sec:nupi} and \ref{sec:experiments} provide a \textit{negative} answer to this question.

\textbf{Dual restarts.} \citet{gallego2022controlled} proposed an approach to mitigate the overshoot in inequality constraints called \textit{dual restarts}: once a constraint is \textit{strictly} satisfied, its associated dual variable is reset to zero. This corresponds to a best response (in game-theoretic terms) of the dual player. Dual restarts prevent excessive enforcement of constraints, which can degrade the achieved objective function value.

However, dual restarts are not suitable for general constrained optimization problems since they rely on determining the satisfaction of the constraint \textit{exactly}. Constraint estimates may (wrongly) indicate strict feasibility due to (i) stochasticity in their estimation, (ii) numerical precision errors making active constraints appear strictly feasible, or (iii) a ``temporary'' strict satisfaction of the constraint. \cref{fig:svm_dynamics_plot} illustrates the undesirable dynamics caused by dual restarts when applied to an SVM task in which the support vectors correspond to strictly active inequality constraints.

In \S \ref{sec:nupi}, we show that \nuPI mitigates the overshoot of inequality constraints, with additional benefits: (i) controllable degree of overshoot (governed by the $\kp$ hyperparameter), (ii) compatibility with equality (and strictly feasible inequality) constraints, and (iii) damping of multiplier oscillations.

\section{$\nu$PI Control for Constrained Optimization}
\label{sec:nupi}

Following \citet{stooke2020responsive}, we consider the learning of an optimal feasible model solving \cref{eq:cmp_definition} as a dynamical system. 
Thus, we can think of the update rule for the multipliers as a control algorithm that aims to \textit{steer the system toward feasibility}. 
We emphasize that we are not trying to control general dynamical systems, but rather systems that arise from partial, inexact minimization (e.g. gradient-based updates) on a min-max Lagrangian game.
In other words, we assume that $\vx_{t-1} \mapsto \vx_{t}$ is updated so as to minimize the current Lagrangian $\Lag(\,\cdot \,, \vtheta_{t})$. 
\Cref{fig:plant_diagram} illustrates the control pipeline we consider throughout this work. 

\vspace{-4ex}

\begin{figure}[ht]
\centering
\begin{tikzpicture}[auto, node distance=0.9cm,>=latex', scale=0.7, every node/.style={scale=0.7}]

    \tikzstyle{block} = [rectangle, draw, text centered, rounded corners, minimum height=4em]
    \tikzstyle{controller} = [block, text width=13em, fill=LightBlue!40]
    \tikzstyle{plant} = [block, text width=16em, fill=Apricot!25] 
    \tikzstyle{measurement} = [block, text width=7em, fill=Apricot!25] 
    \tikzstyle{projection} = [block, text width=5em, fill=LightBlue!40] 
    
    \tikzstyle{circleblock} = [circle, draw, fill=red!25, minimum size=1.5em]
    \tikzstyle{line} = [draw, -latex', line width=0.6pt]

    \node [plant] (plant) {\textbf{Lagrangian dynamics}\\ $\text{approx. } \argmin \vx \Lag(\vx, \vlambda_{t}, \vmu_{t}) \defas f(\vx) + \vlambda_{t}^{\top} \vg(\vx) + \vmu_{t}^{\top} \vh(\vx)$};
    
    \node [measurement, right=of plant, xshift=-3mm] (measurement) {\textbf{Constraint\\measurement} };
    
    \node [controller, above=of $(plant)!0.5!(measurement)$, align=center] (controller2) {\textbf{Inequality Controller} \\$\hat{\vlambda}_{t+1}$ = \nuPI $( \vlambda_{t}, \vg(\vx_{t}), \vlambda_0)$};
    
    \node [controller, below=of $(plant)!0.5!(measurement)$, align=center] (controller1) {\textbf{Equality Controller} \\$\vmu_{t+1}$ = \nuPI $( \vmu_{t}, h(\vx_{t}), \vmu_0)$};

    \node [projection, left=of controller2, xshift=3mm] (projection) {\textbf{Projection} \\ $\ge 0$};

    \node [circleblock, right=of controller1, xshift=1.3cm] (circle1) {};

    \node [circleblock, right=of controller2, xshift=1.3cm] (circle2) {};

    \path [line] (plant) -- node[midway, above] {$\vx_{t}$} (measurement);

    \path [line] (controller2.west) -- node[near start, above, xshift=-2mm] {$\hat{\vlambda}_{t+1}$} (projection);
    
    \path [line] (projection) -- ++(-2cm,0) node[near start, above, xshift=-2mm] {$\vlambda_{t+1}$} |- ([yshift=0.27cm]plant.west) ;
    
    \path [line] (controller1.west) -- node[near start, above] {$\vmu_{t+1}$} ++(-4cm,0) |- ([yshift=-0.25cm]plant.west);
    
    \path[line] ([yshift=-0.25cm]measurement.east) -| node[below, near start] {$\vh(\vx_{t})$}  node[above, near end, xshift=-2mm, yshift=-6mm] {$+$} (circle1);

    \path[line]([yshift=+0.25cm]measurement.east) -| node[above, near start] {$\vg(\vx_{t})$} node[above, near end, xshift=-2mm, yshift=2mm] {$+$}  (circle2);

    \path[line] (circle1)-- node[above, midway] {$\underset{\displaystyle \vh(\vx_{t})-\textbf{0}}{\textbf{Error}}$} (controller1.east);

    \path [line] ([yshift=-5mm]circle1.south) -- node[below, near start] {$\underset{\displaystyle\textbf{level:}\, \mathbf{0}}{\textbf{Target}}$} node[above, near end, xshift= 3mm, yshift=-2mm] {$-$}  (circle1);

    \path[line](circle2)-- node[above, midway] {$\underset{\displaystyle \vg(\vx_{t})-\textbf{0}}{\textbf{Error}}$}(controller2.east);

    \path [line] ([yshift=5mm]circle2.north) -- node[above, near start] {$\underset{\displaystyle\textbf{level:}\, \mathbf{0}}{\textbf{Target}}$} node[below, near end, xshift= 3mm, yshift=2mm] {$-$}  (circle2);

\end{tikzpicture}
\vspace{-4ex}
    \caption[\nuPI control pipeline for updating the Lagrange multipliers in a constrained optimization problem.]{
    \nuPI control pipeline for updating the Lagrange multipliers in a constrained optimization problem.
    \captioncomment{We consider the update on the primal variables as a black-box procedure that receives the multipliers $\vlambda_t$ and primal variables $\vx_{t-1}$ as input, and returns an updated $\vx_{t}.$ The multiplier update is executed by the controller, using the constraint violations as the error signal.}
}
\label{fig:plant_diagram}
\end{figure}
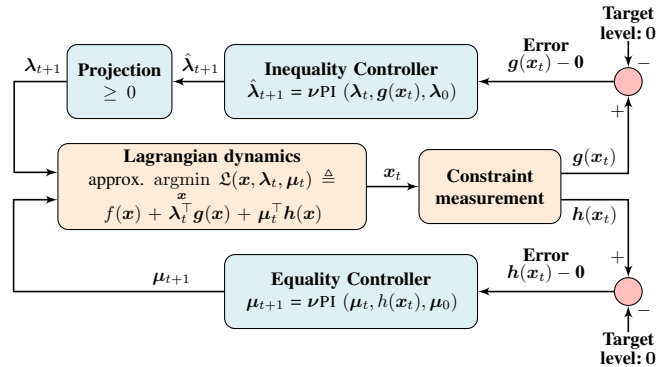

An assumption entailed by the control perspective
in \cref{fig:plant_diagram} is that an increase (resp. decrease) in the control variable (the Lagrange multipliers $\vtheta_t$) leads to a decrease (resp. increase) in the controlled quantity (the constraint violations $\vc(\vx_t)$). 
This assumption holds for constrained optimization problems since an increase in the multipliers leads the primal minimization of the Lagrangian to focus on reducing the value of the constraints 
(as mentioned during the discussion of gradient descent-ascent dynamics in \S \ref{sec:lagrangian_optimization} ).

Note that our black-box assumption on the nature of the primal update allows for an arbitrary choice of optimizer for minimizing $\Lag(\, \cdot \,, \vtheta_{t})$. After obtaining an updated primal iterate $\vx_{t}$, the new constraint violations $\vg(\vx_{t})$ and $\vh(\vx_{t})$ are measured and used as the error signals for the inequality- and equality-constraint controllers, yielding updated multipliers $\vtheta_{t+1}$. The projection block ensures the non-negativity of the multipliers for inequality constraints.

\subsection{$\nu$PI algorithm}

Our main algorithmic contribution is the multiplier update scheme presented in \cref{algo:nupi}. This is a simple generalization of a PI controller (i.e. a PID controller \cite{astrom1995pid} with $\kd = 0$) by including an exponential moving average (of the error signal) in the proportional term. Indeed, the traditional PI controller is recovered when $\nu = 0$.

\begin{algorithm}[h]
   \caption{\nuPI update}
   \label{algo:nupi}
    \begin{algorithmic}
       \STATE {\bfseries Args:} EMA coefficient $\nu$, proportional ($\kp$) and integral ($\ki$) gains; initial conditions $\vxi_{0}$ and $\vtheta_{0}$.
       \STATE \graytext{\small 1:} Measure current system error $\ve_t$
       \STATE \graytext{\small 2:} $\vxi_{t} \leftarrow \nu \vxi_{t-1} + (1 - \nu) \ve_t$ \hfill \graytext{$\triangleright$ {\footnotesize for $t \ge 1$}}
       \STATE \graytext{\small 3:} $\vtheta_{t+1} \leftarrow \vtheta_0 + \kp \vxi_{t} + \ki \sum_{\tau=0}^{t} \ve_{\tau}$
    \end{algorithmic}
\end{algorithm}

The \nuPI update can be equivalently expressed in terms of a recursive update (see \cref{thm:nupi_update_recursive} in \cref{appx:nupi_momentum_connections}) as:
\begin{align}
    \vtheta_1 &= \vtheta_0 + \ki \ve_0 + \kp \vxi_0 \\
    \label{eq:nupi_recursive}
    \vtheta_{t+1} &= \vtheta_{t} +  \ki \ve_t + \kp \left( \vxi_t - \vxi_{t-1}\right) \hspace{2mm} \text{for} \,\, t \ge 1.
\end{align}

\subsection{Connections to optimization methods}
When the error signal corresponds to the negative gradient of a cost function $\ve_t = - \nabla f_t$, \cref{algo:nupi} has straightforward equivalences with common minimization methods (see~\cref{appx:nupi_momentum_connections}). For example, \nuPI$(\nu=0, \kp=0, \ki)$ is equivalent to \algo{GD}~$(\alpha = \ki)$ \citep{stooke2020responsive, lessard2016analysis, an2018pid}. When $\nu = 0$ and $\kp = \ki = \alpha$, \nuPI recovers a single-player version of the \algo{OptimisticGradient} (\algo{OG}) method \citep{popov1980modification}, with step-size $\alpha$.
When $\nu = 0$, but $\kp$ and $\ki$ are allowed to differ, \nuPI coincides with the generalized \algo{OG} studied by \citet{mokhtari2020unified}.
Since we use \nuPI for updating the multipliers, we phrase the updates in \cref{algo:nupi} based on a maximization convention.

Moreover, our proposed algorithm \textbf{\nuPI generalizes popular momentum methods} such as \algo{Polyak}---also known as \algo{HeavyBall}---\citep{polyak1964some} and \algo{Nesterov} \citep{nesterov1983method}.\footnote{We consider a variant of the Nesterov method that uses a constant momentum coefficient.} This connection, stated formally in \cref{thm:um_as_nupi}, will allow us to understand (\S\ref{sec:one_over_x}) why traditional momentum methods are \textit{insufficient} to address the shortcomings of gradient ascent for Lagrangian optimization.

We take advantage of the \algo{UnifiedMomentum} $(\alpha, \beta, \gamma)$ framework introduced by \citet{shen2018unified} to concisely develop a joint analysis of $\algo{Polyak}(\alpha, \beta)=\algo{UM}(\alpha, \beta, \gamma=0)$ and $\algo{Nesterov}(\alpha, \beta)=\algo{UM}(\alpha, \beta, \gamma=1)$.

\begin{algorithm}[h]
   \caption{\algo{UnifiedMomentum} update \citep{shen2018unified}}
   \label{algo:um}
    \begin{algorithmic}
       \STATE {\bfseries Args:} step-size $\alpha$, momentum coefficient $\beta$, interpolation factor $\gamma \in \lspar 0, \frac{1}{1-\beta} \rspar$; initial conditions $\vphi_{0} = \vzero$ and $\vtheta_0$.
       \STATE \graytext{\small 1:} Measure current system error $\ve_t$
       \STATE \graytext{\small 2:} $\vphi_{t+1} \leftarrow \beta \vphi_{t} + \alpha \ve_t$
       \STATE \graytext{\small 3:} $\vtheta_{t+1} \leftarrow \vtheta_{t} + \vphi_{t+1}  + \beta \gamma \left(\vphi_{t+1} - \vphi_{t} \right)$
    \end{algorithmic}
\end{algorithm}

\begin{theorem}{}
\label{thm:um_as_nupi}

    {\normalfont [Proof in \cref{appx:nupi_momentum_connections}.]}
    Under the same initialization $\vtheta_0$, \algo{UnifiedMomentum}$(\alpha, \beta\neq 1, \gamma)$ is a special case of the \nuPI algorithm with the hyperparameter choices:
    \vspace{-1ex}
    \begin{align}
        &\hspace{1mm}\nu \leftarrow \beta 
        \hspace{11mm} \vxi_{0} \leftarrow (1 - \beta) \ve_0 \\
        &\ki \leftarrow \frac{\alpha}{1- \beta} 
        \hspace{4mm} 
        \kp \leftarrow - \frac{\alpha \beta}{(1 - \beta)^2}  \lspar 1 - \gamma ( 1 - \beta) \rspar .
    \end{align}
\end{theorem}

\vspace{-2ex}

\cref{tab:momentum_as_nupi_summary} in \cref{appx:nupi_momentum_connections} summarizes the connections we have established between \nuPI and existing methods. We emphasize that the exponential moving average in \nuPI is a crucial component to obtain the generalization of momentum methods. 

\vspace{-1ex}

\begin{figure}[h]
    \centering
    \hspace{-2mm}
    \includegraphics[trim={1mm 2mm 2mm 0mm}, clip,scale=0.75]{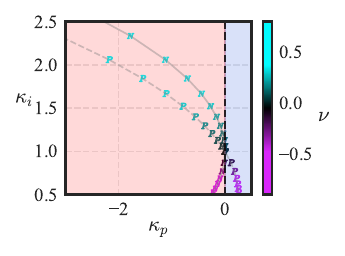}
    \includegraphics[trim={1mm 3mm 5mm 10mm}, clip,scale=0.85]{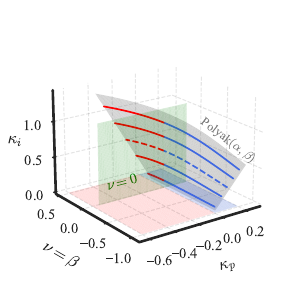}
    \vspace{-2ex}
    \caption[\textbf{Left:} Hyperparameter choices from \cref{thm:um_as_nupi} for which \nuPI$(\nu, \kp, \ki)$ realizes \algo{Polyak}$(\alpha=\frac{1}{2}, \beta)$ and \algo{Nesterov}$(\alpha=\frac{1}{2}, \beta)$. \textbf{Right:} The right plot zooms on the range $-1 \le \beta \le 0.25$.]{
    \textbf{Left:} Hyperparameter choices from \cref{thm:um_as_nupi} for which \nuPI$(\nu, \kp, \ki)$ realizes \algo{Polyak}$(\alpha=\frac{1}{2}, \beta)$ and \algo{Nesterov}$(\alpha=\frac{1}{2}, \beta)$. \textbf{Right:} The right plot zooms on the range $-1 \le \beta \le 0.25$. \textbf{\algo{Polyak} comprises a limited surface in the $(\nu, \kp, \ki)$ space, leaving configurations outside this surface unexplored}. Note how positive (resp. negative) values of $\beta$ result in negative (resp. positive) values of $\kp$, colored in red (resp. blue). 
    \captioncomment{Colored paths correspond to different values of $\alpha$. The dashed curves match between both plots.}
    }
    \label{fig:nupi_sheet}
\end{figure}

\vspace{-1ex}

\cref{fig:nupi_sheet} visually emphasizes the greater generality of \nuPI~compared to \algo{Polyak} and \algo{Nesterov}, presented in \cref{thm:um_as_nupi}. Note that the $\kp$ coefficient changes between \algo{Polyak} and \algo{Nesterov}, while the $\ki$ coefficient coincides. Formally,
\vspace{-1ex}
\begin{equation}
    \label{eq:momentum_betas}
    \hspace{-2mm}
    \kp^\algo{Polyak} = - \frac{\alpha \beta}{(1 - \beta)^2}, \hspace{2mm} \kp^\algo{Nesterov} = - \frac{\alpha \beta^2}{(1 - \beta)^2} \le 0.
\end{equation}

\vspace{-2ex}

Moreover, $\kp^\algo{Nesterov}$ is \textit{non-positive}, regardless of $\beta$. In contrast, a negative momentum value $\beta < 0$ induces a \textit{positive} $\kp^\algo{Polyak}$. This observation is in line with the benefits of using a negative \algo{Polyak} momentum coefficient (for both players) in adversarial games presented by \citet{gidel2019negative}.

\subsection{Interpreting the updates of $\nu$PI}
\label{sec:one_over_x}

Consider the execution of \nuPI$(\nu, \kp, \ki)$ and \GA$(\alpha=\ki)$ at time $t$.\footnote{It is sufficient to consider a single scalar multiplier since the updates of both algorithms decouple across constraints/multipliers.} The relative size between these updates is:
\begin{equation}
    \frac{\Delta \text{\nuPI}}{\Delta \text{\GA}} \defas \frac{\theta^{\nu\text{PI}}_{t+1} -  \theta_t}{\theta^{\text{GA}}_{t+1} -  \theta_t} = \frac{1}{1 - \psi} \lspar 1 -  \frac{\psi \xi_{t-1}}{e_t} \rspar, 
\end{equation}
where $\psi \defas \frac{\kp(1 - \nu)}{\ki + \kp (1 - \nu)}$. \cref{fig:nupi_one_over_x} illustrates the behavior of the relative size of updates of \nuPI compared to \GA. The left plot displays \nuPI with $\kp>0$ and $\nu=0$. The right plot shows the \nuPI-equivalent of \algo{Polyak} with \textit{positive} momentum.\footnote{The case of \algo{Polyak} with negative momentum resembles the left plot of \cref{fig:nupi_one_over_x} See \cref{sec:one_over_x_appx} for further details.}.

Consider the colored regions present in the \textit{left} plot of \cref{fig:nupi_one_over_x}:

{\color[RGB]{180,200,254} \textbf{Mode A}} When $\xi_{t-1} < e_t$, the current violation is greater than the historical violation average (right region). \nuPI algorithm increases the multiplier \textit{faster} than \GA. When $e_t < 0$ (left region), the primal iterate is feasible and the \nuPI algorithm agrees with \GA~in decreasing the multiplier, but does so \textit{much faster} (with a factor above $\frac{1}{1 - \psi}$).

{\color[RGB]{234,215,80} \textbf{Mode B}} When $e_t \in [\psi \xi_{t-1}, \xi_{t-1}]$, the constraint violation has improved compared to the historical average but is still infeasible. In this case, \nuPI increases the multiplier \textit{more slowly} than \GA, consistent with the perceived improvement in the violation.

{\color[RGB]{255,180,180} \textbf{Mode C}} 
When $e_t \in [0, \kappa \xi_{t-1}]$, the primal iterate is still infeasible. However, the \nuPI algorithm determines that the constraint improvement is large enough to warrant a \textit{decrease} in the multiplier. Note that in this case, \GA~would have continued increasing the multiplier.

In all of these cases, the \nuPI optimizer can be seen as executing proactively by considering how the current constraint violation compares to the historical estimates. This proactive behavior allows the method to \textit{increase the multiplier faster than \GA} when the constraint satisfaction is degrading, \textit{and reduce the multiplier faster than \GA} whenever sufficient improvement has been made.

\newpage

\vspace{-1ex}

\begin{figure}[h]
    \centering
    \includegraphics[trim={0.2cm 0 0.2cm 0}, clip, scale=1.05]{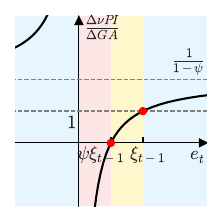}
    \hspace{4mm}
    \includegraphics[trim={0.2cm 0 0.2cm 0}, clip, scale=1.05]{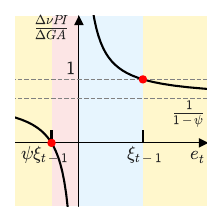}
    \hspace{1mm}
    \vspace{-2ex}
    
    { \scriptsize \hspace*{6mm} $\nu$\textsc{PI}$(\ki, \kp>0, \nu=0)$ \hspace{3mm} $\nu$\textsc{PI}$(\ki, \kp<0, \nu)$=\textsc{Polyak}$(\alpha, \beta>0)$}
    \caption[Comparing the update of \nuPI relative to \GA.]{Comparing the update of \nuPI relative to \GA. \textbf{\nuPI increases the multipliers faster than \GA~when the constraint violation is large, enhancing convergence speed; and proactively decreases them near the feasible set, preventing overshoot.} \captioncomment{The blue, yellow, and red regions correspond to cases in which the updates performed by the \nuPI algorithm are faster, slower, or in the opposite direction than those of \GA, respectively. This plot illustrates the case $\xi_{t-1} > 0$.}
    }
    \label{fig:nupi_one_over_x}
\end{figure}
\vspace{-1ex}

In stark contrast, \cref{fig:nupi_one_over_x} (right) shows a setting in which $\ki$ and $\kp$ have been chosen according to \cref{thm:um_as_nupi} for $\beta=\nu = 0.3$, i.e. using \textit{positive} Polyak momentum. In this case, the algorithm would produce \textit{stronger increases} of the multiplier whenever feasibility is improved, while \textit{weaker increases} are executed whenever feasibility worsens. This counter-intuitive behavior may be the cause of oscillations and overshoot underlying the failure of positive momentum methods in Lagrangian games.  

\subsection{Oscillator dynamics}
\label{sec:oscillator}

The continuous-time dynamics of gradient-descent/$\nu$\algo{PI}-ascent on an equality-constrained problem can be characterized by the second-order differential equations (see \cref{thm:oscillator_flow}):
\vspace{-1ex}
\begin{subnumcases}{}
    & $\ddot{\vx} = \displaystyle - \left(\nabla^2f + \sum_{c'} \mu_{c'} \nabla^2 \vh_{c'} \right) \dot{\vx} - \Jac\vh \dot{\vmu}$  \\
   & $\ddot{\vmu} = \ki \Jac\vh\T \dot{\vx} + \kp \Jac\vh\T \ddot{\vx} + \kp \mXi,$ 
\end{subnumcases}
where $\mXi = \lspar \dot{\vx}\T \nabla^2 h_1 \dot{\vx}, \, \ldots, \, \dot{\vx}\T \nabla^2 h_c \dot{\vx} \rspar\T \in \mathbb{R}^c.$

In \cref{app:oscillator} we present the spectral analysis for the Lagrangian system associated with an equality-constrained quadratic program. In particular, we demonstrate how the continous-time \nuPI algorithm can modify the eigenvalues of the system and transition between divergent, oscillatory, \textit{critically damped} and overdamped behaviors. We show how these regime changes are controlled by the $\kp$ hyperparameter. Moreover, critical damping may require a non-zero value of $\kp$, and is thus not achievable by \GA.

\begin{figure*}[h]
    \vspace{-2ex}
    \centering
    \begin{subfigure}[b]{0.55\textwidth}
        \centering
        \includegraphics[scale=0.9]{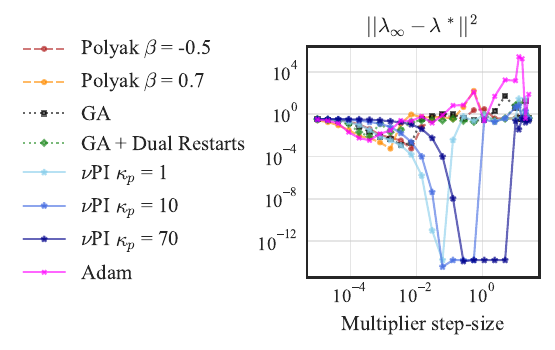}
    \end{subfigure}
    \begin{subfigure}[b]{0.4\textwidth}
        \centering
        \includegraphics[scale=0.9]{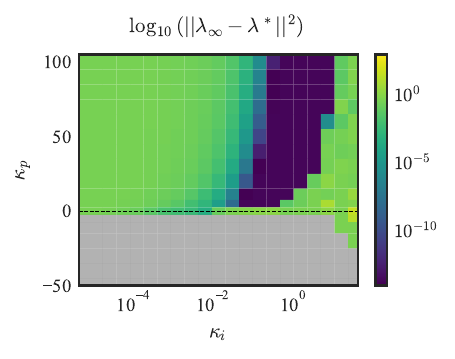}
    \end{subfigure}
    \vspace{-2ex}
    \caption[SVM experiment (\S\ref{sec:svm}). \textbf{Left:} Step-size sensitivity plot. \textbf{Right:} $\kp$ and $\ki$ sensitivity plot for \nuPI.]{SVM experiment (\S\ref{sec:svm}). \textbf{Left:} Step-size sensitivity plot. \textbf{Right:} $\kp$ and $\ki$ sensitivity plot for \nuPI. \textbf{Without a~$\kp$ of at least one, none of the methods converge to the optimal dual variable}. \textbf{Higher $\kp$ values allow for choosing higher and broader range of $\ki$'s.}   \captioncomment{The x-axis of the left plot represents $\ki$ for the \nuPI parameter and $\alpha$, the step-size, for the other optimizers. In the right plot, the gray color shows the runs exceeding a distance of $10^3$ to $\lambda^*$.}}
    \label{fig:svm_sensitivity_analysis}
    \vspace{-2ex}
\end{figure*}

\subsection{Practical remarks}
\label{sec:nupi_practical}

In practice, we suggest the initial condition $\vxi_0 = \ve_0$, as it ensures that the first step of \nuPI matches that of gradient ascent. In cases where the constraints can be evaluated without noise, we suggest a default value of $\nu = 0$. This leaves only the additional hyperparameter $\kp$ to be tuned (besides the ``step-size'' $\ki$). We highlight that the main benefits of the \nuPI algorithm remain even when $\nu = 0$. However, $\nu$ can be useful for filtering noise in the constraint measurement, as shown in our fairness experiments in \S\ref{sec:fairness}.

There is a predictable monotonic behavior of the damping of the system as the $\kp$ coefficient increases. This is illustrated in  \cref{fig:main_ablation} in~\S\ref{sec:sparsity} for a sparsity task. As a side effect, higher values of $\kp$ make the tuning of the $\ki$ coefficient easier, as seen in \cref{fig:svm_sensitivity_analysis} in~\S\ref{sec:svm}.
As a heuristic to tune \nuPI, we suggest considering a large initial $\kp$ value (so that its influence on the optimization dynamics is significant), and then try a grid of $\ki$ values. A good starting place is a grid of $\ki$ values around a suitable step-size for gradient ascent.

\section{Experiments}
\label{sec:experiments}

In this section, we present an empirical comparison between \nuPI and a series of baseline optimization methods popular for minimization. We consider gradient ascent, gradient ascent with positive \citep{polyak1964some,nesterov1983method} and negative \citep{gidel2019negative} momentum, and \algo{Adam} \citep{kingma2014adam}. The goal of our experiments is to highlight the flexibility of \nuPI and its ability to mitigate oscillations and overshoot when used to optimize Lagrange multipliers. 

Our implementations use PyTorch \citep{pytorch} and the Cooper library for Lagrangian constrained optimization \citep{gallegoPosada2022cooper}.

\subsection{Hard-margin SVMs}
\label{sec:svm}

We consider solving a \textit{hard-margin} linear SVM problem via its associated Lagrangian formulation. While specialized QP solvers exist to find solutions for this task, we consider the Lagrangian formulation in order to illustrate the dynamics of the multipliers in a simple machine learning task. These experiments show how standard methods for minimization produce oscillations on the multipliers, which have detrimental effects on convergence. Consider
\begin{equation}
    \label{eq:svm_problem}
    \underset{\vw}{\text{min}} \, \|\vw\|^2/2  \hspace{2mm}
    \text{s.t.}  \hspace{2mm}  y_i  (\vw^\top \vx_i + b) \geq 1 \hspace{2mm} \text{for } i\in [m],     
\end{equation}
where $\{(\vx_i, y_i)\}_{i=1}^m$ are labeled training datapoints, and $\vw$ and $b$ are the parameters of the SVM classifier. 

We perform binary classification on two linearly separable classes from the Iris dataset \cite{misc_iris_53}. We apply alternating GDA updates on the Lagrangian associated with \cref{eq:svm_problem}, with a fixed optimizer for the primal variables.
For details on our SVM experiments, see \cref{app:svm}.

\textbf{Multiplier dynamics.} \cref{fig:svm_dynamics_plot} shows the oscillations on the multiplier in all of the baselines. In these tasks, all of the methods that do not diverge achieve perfect training and validation accuracy. However, among the methods we experiment with, the only method capable of achieving zero constraint violation is the \nuPI algorithm. In contrast to all baselines, \nuPI dampens the oscillations and converges to the optimal multiplier value. 

\textbf{Sensitivity analysis.} \cref{fig:svm_sensitivity_analysis} (left) illustrates the robustness of \nuPI to the choice of $\ki$. The considered baselines fail to converge to the ground truth multiplier value, across a wide range of step-sizes. For these baselines, small step-size choices avoid divergence but do not lead to recovering the optimal Lagrange multipliers, while large step-sizes increase the oscillations.
In contrast, introducing a $\kp$ term of more than 1 results in convergence for some step-sizes within the selected range (see the \nuPI curves). Moreover, increasing $\kp$ to a higher value broadens the range of step-sizes that lead to convergence, and enables the use of bigger step-size values that converge. This behavior can be observed more extensively in the heatmap of \cref{fig:svm_sensitivity_analysis} (right).

\subsection{Fairness}
\label{sec:fairness}

We consider a classification task under \textit{statistical parity} constraints, as described in \citet{cotter2019proxy}.
This leads to the following constrained optimization problem:
\begin{equation}
    \label{eq:fairness_problem}
    \underset{\vw}{\text{min}} \, L(\vw)  \hspace{2mm}
    \text{s.t.}  \hspace{2mm}  \text{P}(\hat{y} = 1 \, | \, g) = \text{P}(\hat{y} = 1), \quad \forall g  \in G  
\end{equation}
where $L(\vw)$ is the loss of model $\vw$, $\hat{y}$ is the model prediction, and $G$ represents the set of protected groups in the dataset. The constraints require the probability of positive prediction to be equal across all groups. 

\textbf{Model and data.} We train binary classifiers on the Adult dataset \cite{misc_adult_2}. Groups correspond to the intersection of race (2 values) and gender (5 values), leading to 10 constraints. We use an MLP with two 100-dimensional hidden layers.
Our experimental setup is similar to those of \citet{zafar2019fairness} and \citet{cotter2019proxy}.
However, in our setting, non-convexity precludes the use of specialized solvers (as done by \citet{zafar2019fairness}) and requires iterative optimization approaches.

\textbf{Optimization configuration.} We train the model using \Adam~($\alpha = 10^{-2}$) with a batch size of 512. To mitigate the noise in the estimation of the constraint satisfaction, we update the multipliers once every epoch, using the exact constraint measurement over the entire training set.

\textbf{Results.} \Cref{fig:fairness_plot} includes training curves for experiments with \GA~and \nuPI applied to the Lagrange multipliers. We report two of the multipliers, the model accuracy, and the maximum constraint violation (in absolute value). 

For this task, gradient ascent is a strong baseline as it successfully reduces the violation of the constraints. Both \GA~and \nuPI($\nu=0.99$) significantly improve compared to an \textit{unconstrained} baseline which achieves a maximum violation of 20\% (not shown in \cref{fig:fairness_plot} for readability). 

\nuPI($\nu=0$) runs exhibit unstable multiplier dynamics as the noise of the constraints is amplified by the proportional term. 
During our experiments, we observed that when $\nu=0$, larger $\kp$ values lead to noisier multipliers and unstable optimization.
In contrast, \textbf{\nuPI($\nu=0.99$) reduces the maximum violation faster and achieves better training accuracy (92.4\% vs 89\%).}  

All experiments reach a final maximum violation of around 1.7\%. We hypothesize that it is not possible to decrease this value further (while carrying out stochastic updates on the primal variables) since the constraint gradients may be misaligned across mini-batches. 

\textbf{Multiplier dynamics.} As can be seen in the evolution of multipliers 2 and 7 shown in \cref{fig:fairness_plot}, \nuPI yields multipliers that stabilize at their limiting values faster than those of \GA. 

\begin{figure}[h]
    \vspace{1ex}
    \centering
    \hspace*{-2ex} \includegraphics[scale=0.88]{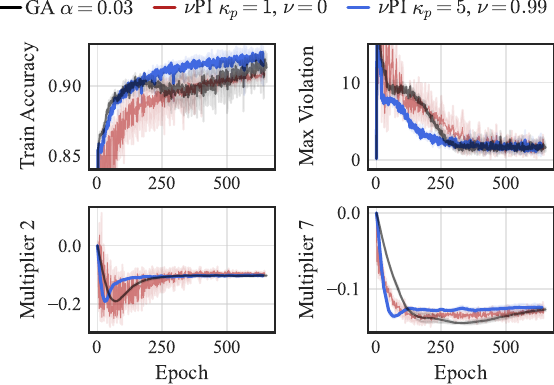}
    \caption[Dynamics of \nuPI compared to \GA~for the fairness task in\cref{eq:fairness_problem}.]{Dynamics of \nuPI compared to \GA~for the fairness task in \cref{eq:fairness_problem}. \textbf{\nuPI has faster convergence in the multiplier value and achieves a better training accuracy than \GA.}
    \captioncomment{All dual optimizers use a step-size ($\ki$ for \nuPI) of 0.03.
    Results are aggregated across five seeds.}
    }
    \label{fig:fairness_plot}
    \vspace*{-2ex}
\end{figure}

\subsection{Sparsity}
\label{sec:sparsity}

We consider the problem of learning models under inequality $L_0$-sparsity constraints \citep{louizos2017learning,gallego2022controlled}.
See \cref{app:details:sparsity} for further background.
\begin{equation}
    \label{eq:sparsity_cmp}
    \min_{\vw, \vphi \in \mathbb{R}^d} \mathbb{E}_{\vz | \vphi} \left[ L(\vw \odot \vz \,|\, \mathcal{D}) \right] \quad \text{s.t. } \frac{\mathbb{E}_{\vz | \vphi} [\|\vz\|_0]}{\#(\vw)} \leq \epsilon
\end{equation}

When using \GA~updates for the multipliers, \citet{gallego2022controlled} observe a tendency of the model to ``overshoot'' into the feasible region and become significantly less dense than the prescribed level. Since a reduction in model density corresponds to a reduction in capacity, this overshoot may have a detrimental effect on the performance of the model.

Our experiments explore the effect of \nuPI on the sparsity-constrained task, and compare it with dual restarts \citep[\S \ref{sec:lagrangian_optimization}]{gallego2022controlled}. Our results show that \nuPI allows for fine-grained control over overshoot, thus enabling the sparse model to retain as much performance as possible.

\textbf{Experiment configuration and hyperparameters.} We consider classifying CIFAR-10 \citep{cifar} images with ResNet-18 \citep{he2016deep} models. To highlight the ease-of-use of \nuPI, our setup remains as close as possible to \citet{gallego2022controlled}: we apply output channel sparsity on the first layer of each residual block in the model, and re-use the authors' choice of optimizer and step-size for $\vphi$. Our sparsity experiments consider $\nu = 0$. 

\textbf{Global and layer-wise settings.} We present sparsity experiments with either \blobletter{1} one global constraint on the sparsity of the entire model, or \blobletter{2} multiple constraints, each prescribing a maximum density per layer. 

The metrics reported in this section are aggregated across 5 seeds. Experimental details for this task can be found in \cref{app:details:sparsity}.
For comprehensive experimental results across multiple sparsity levels, and ablations on the use of momentum and \Adam~for updating the multipliers, see \cref{app:experiments:sparsity}.

\textbf{Results.} \cref{fig:main_tradeoff} shows how gradient ascent and positive and negative momentum values consistently yield runs that overshoot into becoming overly sparse. The extra reduction in capacity results in a loss in performance. In contrast, \nuPI consistently recovers feasible solutions, with minimal overshoot. While dual restarts do not incur in overshoot, they produce slightly infeasible solutions.

\begin{figure}[h]
    \centering
    \includegraphics[scale=1]{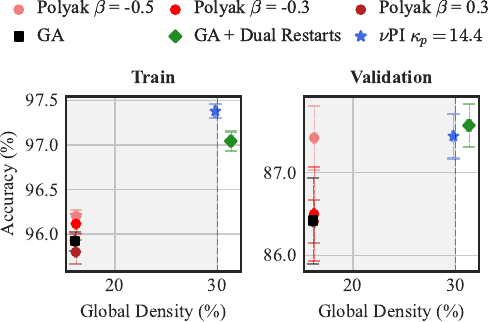}
    \caption[CIFAR10 trade-off plot for \textit{global} sparsity under a 30\% density target.]{CIFAR10 trade-off plot for \textit{global} sparsity under a 30\% density target. \textbf{\nuPI successfully achieves the desired sparsity while achieving the highest train accuracy.}
    \captioncomment{The shaded region is the feasible set. As higher density correlates to higher train accuracy, overshooting to a lower density is undesirable. All optimizers use the same step-size.}
    }
    \label{fig:main_tradeoff}
    \vspace{-1ex}
\end{figure}

\cref{fig:main_ablation} consists on an ablation on the $\kp$ value. We observe that by increasing the hyper-parameter, overshoot is reduced, eventually turning into undershoot (which leads to infeasible solutions). Since the density of the model is monotonically tied to the choice of $\kp$, tuning \nuPI for this task can be done via bisection search, without the need to consider a grid (which is usually required for tuning the step-size). 

\cref{table*:layerwise_sparsity} shows sparsity experiments with layer-wise sparsity targets. Gradient ascent and momentum methods overshoot and the degree of overshoot differs significantly across layers. In contrast, \GA~with dual restarts and \nuPI mitigate overshoot and produce constraints spanning a narrow range of values. 
This highlights the robustness of \nuPI as the $\kp$ coefficient did not need to be tuned separately per constraint.

\textbf{Multiplier dynamics.} \Cref{fig:dynamics_ga} shows the training dynamics for a global sparsity constraint and its multiplier under a 30\% density target. 
We observe that \GA~and \Polyak~quickly lead to overshoot into the feasible set, but manage to regain some model capacity as training progresses. 
\GA~with dual restarts sets the value of the multiplier to zero as soon as feasibility is achieved, thus preventing an incursion of the constraint into the feasible set. \nuPI produces well-behaved multipliers and successfully avoids constraint overshoot.

\begin{figure}[t]
    \centering
    \begin{subfigure}[b]{0.4\textwidth}
        \centering
        \includegraphics[width=1\textwidth]{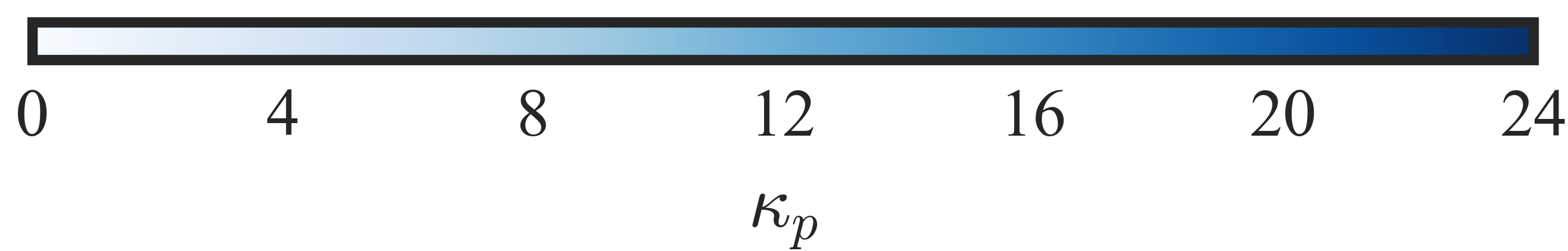}
    \end{subfigure}
    \begin{subfigure}[b]{0.48\textwidth}
        \centering
        \includegraphics[scale=1]{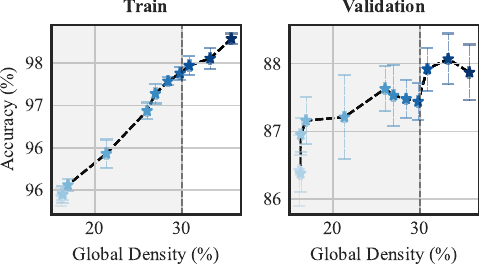}
    \end{subfigure}
    \caption[Ablation of $\kp$ values for \nuPI on CIFAR10.]{Ablation of $\kp$ values for \nuPI on CIFAR10. \textbf{An increasing $\kp$ leads to more damping and less overshoot.}
    \captioncomment{Target density is 30\%. The shaded region is the feasible set.}
    }
    \label{fig:main_ablation}
\end{figure}
\vspace{-1ex}

\begin{table}[h]
\centering
\caption[CIFAR10 results for \textit{layer-wise} sparsity under a $30\%$ density target.]{CIFAR10 results for \textit{layer-wise} sparsity under a $30\%$ density target. \algo{GA} and momentum methods overshoot to \textit{different} values for each constraint. 
\textbf{\nuPI achieves the desired sparsity \textit{on all layers} while achieving the highest train accuracy.} \captioncomment{All dual optimizers use the same step-size.}}
\label{table*:layerwise_sparsity}
\begin{adjustbox}{max width=0.45\textwidth}
\begin{tabular}{l|cc|ccc}
\toprule
\multirow{2}{*}{Method} & \multicolumn{2}{|c}{Accuracy} & \multicolumn{3}{|c}{Violation} \\
 & \multicolumn{1}{|c}{Train} & \multicolumn{1}{c|}{Test} & \multicolumn{1}{|c}{Min} & \multicolumn{1}{c}{Max} & \multicolumn{1}{c}{Range} \\
\midrule
Polyak $\beta = -0.5$ & 91.9 & 83.6 & -26.5 & -7.9 & 18.9 \\
Polyak $\beta = -0.3$ & 92.1 & 83.4 & -27.1 & -6.7 & 20.6 \\
Polyak $\beta = 0.3$ & 91.9 & 82.5 & -26.3 & -2.3 & 24.0 \\
GA & 92.0 & 84.1 & -27.8 & -5.2 & 22.0 \\
GA + Dual Restarts & 95.0 & 85.3 & -0.0 & 1.2 & 1.2 \\
\textit{Ours} - $\nu$PI $\kappa_p = 8.0$ & 95.1 & 86.2 & -1.7 & 0.1 & 1.8 \\
\bottomrule
\end{tabular}
\end{adjustbox}
\end{table}

\section{Conclusion}

Previous work has highlighted that employing PID controllers on the multipliers in Lagrangian constrained optimization problems reduces oscillation and overshoot. 
In this paper, we consider \nuPI, a variant of a PI controller that generalizes various popular methods for optimization. 
We complement previous work by providing insights justifying why PI controllers are desirable for Lagrangian optimization. Moreover, we highlight some intuitions as to why momentum methods fail in this context. 
While we focus our efforts on constrained optimization, our results indicate that \nuPI may improve the dynamics of linear players in general min-max games. 
Investigating the behavior of \nuPI on non-linear players is left as a direction of future work.

\newpage

\section*{Impact Statement}

Constrained optimization offers tools for reliably enforcing properties on machine learning models. It is, therefore, applicable for enhancing safety, robustness, and fairness in AI models.
By integrating constraints into the model development process, rather than retrofitting safety measures as afterthoughts, we advocate for a paradigm shift towards building models that are inherently secure ``by design." 
We intend our fairness experiments as a conceptual illustration of the potential for positive impact of constrained approaches in the development of machine learning models.

Our paper presents insights into the robustness of algorithms for constrained optimization, and highlights \nuPI as a reliable tool for training models with constraints. Thus, our work lays the groundwork for practitioners to adopt and implement constrained approaches confidently in diverse real-world applications.

\section*{Acknowledgements}

This research was partially supported by an IVADO PhD Excellence Scholarship, the Canada CIFAR AI Chair program (Mila), the NSERC Discovery Grant RGPIN2017-06936 ahd by Samsung Electronics Co., Ldt. Simon Lacoste-Julien is a CIFAR Associate Fellow in the Learning in Machines \& Brains program.

This research was enabled in part by compute resources, software, and technical help provided by Mila. %

We would like to thank Ioannis Mitliagkas for useful discussions during the development of this work.

\section*{Reproducibility Statement}

We provide our code,\footnote{\href{https://github.com/motahareh-sohrabi/nuPI}{\texttt{https://github.com/motahareh-sohrabi/nuPI}}} including scripts to replicate the experiments in this paper.
\S \ref{sec:nupi_practical} presents some considerations when using the \nuPI algorithm in practice. 
Experimental details, as well as the hyper-parameters used in our experiments, are included in \cref{app:exp_details}. 
Our implementations use the open-source libraries PyTorch~\citep{pytorch} and Cooper~\citep{gallegoPosada2022cooper}.

\bibliography{references}
\bibliographystyle{icml2024_template/icml2024}

\newpage
\appendix
\onecolumn

\addcontentsline{toc}{section}{Appendix}
\part{Appendix} 
\parttoc
\newpage

\section{Further discussion on prior works using PID controls in optimization}
\label{appx:pid_related_works}

\begin{itemize}
    \item In \citet{stooke2020responsive}, the authors focus almost exclusively on applying PID control to constrained reinforcement learning. The authors do not explore the optimization aspects of PID-based updates for Lagrange multipliers, which are the main focus of our work. Our key theoretical contribution (\cref{thm:um_as_nupi}) shows that \nuPI provides a generalization of momentum-based optimization techniques. Note that the controller considered by \citet{stooke2020responsive} is unable to generalize momentum methods. Thanks to the unifying framework provided by \cref{thm:um_as_nupi}, we provide insights to understand why momentum fails at Lagrangian optimization tasks (\cref{fig:nupi_one_over_x}). Moreover, our experiments encompass linear SVMs, sparsity, and fairness tasks, and are not restricted to reinforcement learning.
    \item \citet{an2018pid} propose directly updating the parameters of a neural network using a PID controller (for unconstrained minimization only). Their approach has not been widely adopted by the deep learning community, possibly due to the highly specialized training procedures that have been developed for training neural networks. Although connected due to their use of PID control, this paper is not directly relevant to our work as we limit our scope to not modifying the optimization protocol for the (primal) model parameters.
    \item The work of \citet{casti2023control} focuses on the theoretical aspects of using PID control for problems with linear constraints. Their analysis is not directly applicable to our setting since we are interested in general machine learning applications involving nonconvex constraints.
    \item \citet{hu2017control} present control interpretations of first-order optimization methods and show how worst-case convergence rates of optimization algorithms can be derived from a control theoretical perspective. The idea of examining a possible connection between our \nuPI~algorithm and other momentum methods was inspired by this work.    
\end{itemize}

\section{Connections between $\nu$PI and momentum methods}
\label{appx:nupi_momentum_connections}

\begin{table}[h]
\centering
\caption{Classical optimization methods as 
instances of \nuPI$(\nu,\kp, \ki; \vxi_{0})$.}
\label{tab:momentum_as_nupi_summary}
\renewcommand{\arraystretch}{1.3}
\begin{tabular}{c|c|ccc}
\toprule
\textbf{Algorithm}  & $\vxi_{0}$ & $\kp$ & $\ki$ & $\nu$  \\
\hline
\algo{UnifiedMomentum}$(\alpha, \beta, \gamma)$ & $\displaystyle (1 - \beta) \ve_0$  &$\displaystyle  - \frac{\alpha \beta}{(1 - \beta)^2}  \lspar 1 - \gamma ( 1 - \beta) \rspar$ & $\displaystyle \frac{\alpha}{1- \beta} $ & $\beta$  \\
\hline

\algo{Polyak}$(\alpha, \beta)$ & $\displaystyle (1 - \beta) \ve_0$ & $\displaystyle - \frac{\alpha \beta}{(1 - \beta)^2} $ & $\displaystyle  \frac{\alpha}{1 - \beta}$& $\beta$ \\
\hline

\algo{Nesterov}$(\alpha, \beta)$ & $\displaystyle (1 - \beta) \ve_0$  & $\displaystyle - \frac{\alpha \beta^2}{(1 - \beta)^2}$ & $\displaystyle  \frac{\alpha}{1 - \beta}$& $\beta$  \\
\hline

\algo{PI} & $\displaystyle \ve_0$ & $\displaystyle \kp $ & $\displaystyle  \ki $& 0 \\
\hline

\algo{OptimisticGradientAscent}$(\alpha)$ & $\displaystyle \ve_0$ & $\displaystyle \alpha$ & $\displaystyle \alpha $ & 0 \\
\hline

\algo{\nuPI}$(\nu,\kp, \ki)$ in practice & $\vzero$ & $\ki$ & $\kp$ & $\nu$  \\

\hline

\algo{GradientAscent}$(\alpha)$ & -- & 0 & $\alpha$ & 0  \\

\bottomrule
\end{tabular}
\end{table}

\begin{lemma}{}
\label{thm:nupi_update_recursive}

    The \nuPI$(\nu, \kp, \ki; \vxi_{0})$ algorithm can be equivalently expressed as the recursion:
    \begin{subequations}
        \begin{align}
            \vtheta_1 &= \vtheta_0 + \ki \ve_0 + \kp \vxi_0 \label{eq:nupi_recursive_first} \\
            \vxi_t &= \nu \vxi_{t-1} + (1 - \nu) \ve_t \label{eq:nupi_ksi} \\
            \vtheta_{t+1} &= \vtheta_{t} +  \ki \ve_t + \kp (1-\nu) \left( \ve_t - \vxi_{t-1}\right)   \,\, \text{for } t \ge 1 \label{eq:nupi_recursive_inductive}
        \end{align}
    \end{subequations}
\end{lemma}

\begin{proof}[\textbf{Proof of \cref{thm:nupi_update_recursive}}]
    For $t \ge 1$, we have:
    \begin{equation}
        \vtheta_{t+1} - \vtheta_t \overset{(\nu \text{PI} 3)}{=} \kp \vxi_{t}  + \ki \sum_{\tau=0}^{t} \ve_{\tau} - \kp \vxi_{t-1} - \ki \sum_{\tau=0}^{t-1} \ve_{\tau}  = \ki \ve_t + \kp ( \vxi_t - \vxi_{t-1}) \overset{(\nu \text{PI} 2
        )}{=} \ki \ve_t + \kp (1 - \nu) (\ve_{t} - \vxi_{t-1})
    \end{equation}
\end{proof}

\begin{lemma}{}
\label{thm:um_update_recursive}

    The \algo{UnifiedMomentum}$(\alpha, \beta, \gamma; \vphi_0 = \vzero)$ algorithm can be expressed as the single-parameter recurrence:
    \begin{subequations}
        \begin{align}
            \vtheta_1 &= \vtheta_0 + \alpha (1 + \beta \gamma) \ve_0 
            \label{eq:um_recursive_first} \\
            \vtheta_{t+1} &= \vtheta_t + \alpha \ve_t + \beta (\vtheta_t - \vtheta_{t-1}) + \alpha \beta \gamma \left(\ve_t - \ve_{t-1} \right) \,\, \text{for } t \ge 1. \label{eq:um_recursive_inductive} 
        \end{align}
    \end{subequations}
\end{lemma}

\begin{proof}[\textbf{Proof of \cref{thm:um_update_recursive}}]

    \begin{subequations}
        \begin{align}
            \vphi_1 &\overset{(\textsc{UM}_2)}{=} \beta \cancel{\vphi_0} + \alpha \ve_0 = \alpha \ve_0 \\ 
            \vtheta_{1} &\overset{(\textsc{UM}_3)}{=} \vtheta_{0} + \vphi_{1}  + \beta \gamma \left(\vphi_{1} - \cancel{\vphi_{0}} \right)  = \vtheta_0 + \alpha ( 1 + \beta \gamma) \ve_0.
        \end{align}
    \end{subequations}
    
    \vspace{-4ex}
    
    \begin{subequations}
        \begin{align}
            \vtheta_{t+1} &\overset{(\textsc{UM}_3)}{=} \vtheta_{t} + \vphi_{t+1}  + \beta \gamma \left(\vphi_{t+1} - \vphi_{t} \right) \\
            &\overset{(\textsc{UM}_2)}{=} \vtheta_{t} + \beta \vphi_t + \alpha \ve_t + \beta \gamma \left( \beta \vphi_t + \alpha \ve_t - \beta \vphi_{t-1} - \alpha \ve_{t-1} \right) \\
            &\,\,\,\,=\,\,\,\, \vtheta_{t} + \beta \vphi_t + \alpha \ve_t + \beta \gamma \left( \beta (\vphi_t - \vphi_{t-1}) + \alpha (\ve_t - \ve_{t-1}) \right) \\
            &\,\,\,\,=\,\,\,\, \vtheta_{t} + \alpha \ve_t + \beta \left[ \vphi_t + \gamma \beta (\vphi_t - \vphi_{t-1}) \right] + \alpha \beta \gamma (\ve_t - \ve_{t-1}) \\
            &\overset{(\textsc{UM}_3)}{=} \vtheta_{t} + \alpha \ve_t + \beta \left( \vtheta_t - \vtheta_{t-1} \right) + \alpha \beta \gamma (\ve_t - \ve_{t-1}).
        \end{align}
    \end{subequations}
\end{proof}

\begin{nonumtheorem}{\normalfont \textbf{\ref{thm:um_as_nupi}}}
    Under the same initialization $\vtheta_0$, \algo{UnifiedMomentum}$(\alpha, \beta \neq 1, \gamma; \vphi_{0} = \vzero)$ is a special case of \nuPI$(\nu, \kp, \ki; \vxi_{0})$ with the following hyperparameter choices:
    \begin{equation}
        \nu= \beta 
        \hspace{7mm}
        \kp = - \frac{\alpha \beta}{(1 - \beta)^2}  \lspar 1 - \gamma ( 1 - \beta) \rspar 
        \hspace{7mm}
        \ki = \frac{\alpha}{1- \beta} 
        \hspace{7mm}
        \vxi_{0} = (1 - \beta) \ve_0
    \end{equation}
\end{nonumtheorem}

\begin{proof}[\textbf{Proof of \cref{thm:um_as_nupi}}]
    We want to find values of $\nu$, $\kp$, $\ki$ and $\vxi_{0}$ such that the sequence of iterates produced by \nuPI$(\nu, \kp, \ki; \vxi_{0})$ satisfies \cref{eq:um_recursive_inductive}. For $t \ge 2$ we have:
    \begin{subequations}
        \begin{align}
            \vtheta_{t+1} - \vtheta_t & \overset{  (\ref{eq:um_recursive_inductive})}{=} \alpha \ve_t + \beta (\vtheta_{t} - \vtheta_{t-1}) + \alpha \beta \gamma \left( \ve_t - \ve_{t-1} \right) \\
            \ki \ve_t + \kp \left( \vxi_t - \vxi_{t-1}\right) & \overset{(\ref{eq:nupi_recursive_inductive})}{=} \alpha \ve_t + \beta  \lpar \ki \ve_{t-1} + \kp \left( \vxi_{t-1} - \vxi_{t-2}\right)  \rpar + \alpha \beta \gamma \left( \ve_t - \ve_{t-1} \right)
        \end{align}
    \end{subequations}

\begin{equation}
    \ve_t \lpar \ki - \alpha - \alpha \beta \gamma \rpar 
    + 
    \ve_{t-1} \lpar -\beta \ki + \alpha \beta \gamma \rpar
    +  
    \kp \lspar \vxi_t - \vxi_{t-1} - \beta  \lpar \vxi_{t-1} - \vxi_{t-2} \rpar \rspar
    = 0
\end{equation}

Several applications of the definition of $\vxi_t$ (line 2 in \cref{algo:nupi}) give:
\begin{align}
    \vxi_t - \vxi_{t-1} - \beta  \lpar \vxi_{t-1} - \vxi_{t-2}\rpar &= (1 - \nu) \lspar \ve_t - \vxi_{t-1} \rspar - \beta \vxi_{t-1} + \beta \vxi_{t-2} \\
    &= (1-\nu) \ve_t - (1 + \beta - \nu) \vxi_{t-1} + \beta \vxi_{t-2} \\
    &= (1-\nu) \ve_t - (1 + \beta - \nu) \lspar \nu \vxi_{t-2} + (1 - \nu) \ve_{t-1} \rspar + \beta \vxi_{t-2} \\
    &= (1-\nu) \ve_t - (1 - \nu) (1 + \beta - \nu) \ve_{t-1} + (1 - \nu) (\beta - \nu) \vxi_{t-2} 
\end{align}

Thus we can re-arrange to get:
    \begin{equation}
        \begin{bmatrix} \ve_t & \ve_{t-1} & \vxi_{t-2} \end{bmatrix}
        \begin{bmatrix}
            \ki + (1-\nu) \kp - \alpha (1 + \beta \gamma) \\
            -\beta \ki - (1 - \nu)(1 + \beta - \nu) \kp  + \alpha \beta \gamma \\
            (1 - \nu)(\beta - \nu) \kp 
        \end{bmatrix} = 0
    \end{equation}
    
    Therefore, from 
    both algorithms coincide when the following system of equations is satisfied:
    \begin{subequations}
        \begin{align}
            \ki + (1-\nu) \kp &= \alpha (1 + \beta \gamma) \label{eq:nupi_um_system} \\
            \beta \ki + (1 - \nu)(1 - \nu + \beta) \kp  &= \alpha \beta \gamma \\
            (1 - \nu) (\beta - \nu) \kp &= 0
        \end{align}
    \end{subequations}

    For $\beta \neq 1$, the solution to this system is given by:
    \begin{equation}
        \textcolor{mathgreen}{
            \nu \leftarrow \beta 
            \hspace{9mm}
            \ki \leftarrow \frac{\alpha}{1- \beta} 
            \hspace{9mm}
            \kp \leftarrow - \frac{\alpha \beta}{(1 - \beta)^2}  \lspar 1 - \gamma ( 1 - \beta) \rspar
        }
        \label{eq:um_nupi_equivalence_hyperparameters}
    \end{equation}
    Finally, we choose the initial condition $\vxi_{0}$ that ensures that the first two steps of the algorithms match (at $t=0$ and $t=1$). The first iterate of \nuPI is given by $\vtheta_1 = \vtheta_0 + \ki \ve_0 + \kp \vxi_0 $ as per \cref{eq:nupi_recursive_first}. Meanwhile, the first iterate of \algo{UnifiedMomentum} is given by:
    \begin{equation}
        \vtheta_1 \overset{(\ref{eq:um_recursive_first})}{=} \vtheta_0 + \alpha (1 + \beta \gamma) \ve_0 \overset{(\ref{eq:um_nupi_equivalence_hyperparameters})}{=} \vtheta_0 + (\ki + (1 - \beta) \kp) \ve_0  = \vtheta_0 + \ki \ve_0 + (1 - \beta) \kp \ve_0
    \end{equation}
    Therefore, setting ${ \color{mathgreen} \vxi_{0} \leftarrow (1 - \beta ) \ve_0 }$ makes both algorithms match in their first step at $t=0$.  

    The second iterate from \algo{UnifiedMomentum} is $\vtheta_2 = \vtheta_1 + \alpha \beta \lspar 1 - \gamma (1 - \beta) \rspar \ve_0 + \alpha \lspar 1 - \gamma (1 - \beta) \rspar \ve_1$. On the other hand, the second iterate of \nuPI is $\vtheta_2 = \vtheta_1 + (\ki + \kp (1-\nu)) \ve_1 - \kp ( 1- \beta) \vxi_0$. It is easy to see that, given the hyperparameter choices outlined above, both algorithms match at $t=1$.

    An induction argument yields the equivalence between the algorithms.    
\end{proof}

\section{Interpreting the updates of $\nu$PI}
\label{sec:one_over_x_appx}

Consider the execution of the algorithms \nuPI$(\nu, \kp, \ki)$ and \algo{GA}$(\alpha = \ki)$ at time $t$, with updates given by: 
\begin{align}
    \theta^{\nu\text{PI}}_{t+1} &=  \theta_t + \ki e_t + \kp ( 1 - \nu) (e_t - \xi_{t-1}) \\   
    \theta^{\text{GA}}_{t+1} &=  \theta_t + \ki e_t
\end{align}

Let $\psi = \frac{\kp(1 - \nu)}{\ki + \kp (1 - \nu)}$. Note that whenever $\kp$ and $\ki$ are non-negative, $\psi \in [0,1]$. The ratio between these updates is:
\begin{equation}
    \label{eq:relative_update}
    \frac{\Delta \text{\nuPI}}{\Delta \text{GA}} = \frac{\theta^{\nu\text{PI}}_{t+1} -  \theta_t}{\theta^{\text{GA}}_{t+1} -  \theta_t} 
    = \frac{\ki e_t + \kp ( 1 - \nu) (e_t - \xi_{t-1})}{\ki e_t} 
    =  1 + \frac{\kp(1 - \nu)}{\ki}
    - \frac{\kp (1- \nu)}{\ki} \frac{\xi_{t-1}}{e_t} = \frac{1}{1 - \psi} \lspar 1 -  \frac{\psi \xi_{t-1}}{e_t} \rspar.
\end{equation}

\begin{figure}[h]
    \centering
    \includegraphics[scale=1.1]{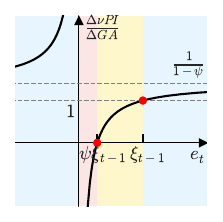}
    \includegraphics[scale=1.1]{figures/one_over_x/one_over_x_kp_1_ki_1_nu_0.pdf}
    \includegraphics[scale=1.1]{figures/one_over_x/one_over_x_positive_momentum_0.3.pdf}
    \caption[Comparing the update of \nuPI~relative to \GA.]{Comparing the update of \nuPI relative to \GA, for different hyper-parameter configurations of \nuPI. 
    \textbf{Left:} \nuPI is configured to recover \algo{Polyak}$(\beta=-0.4$). Updates exhibit dampening similar to that of \nuPI$(\nu=0, \kp=1, \ki=1)$. 
    \textbf{Middle:} \nuPI$(\nu=0, \kp=1, \ki=1)$ corresponding to a \algo{PI} controller. 
    \textbf{\nuPI~increases the multipliers faster than \GA~when the constraint violation is large, enhancing convergence speed; and proactively decreases them near the feasible set, preventing overshoot.}
    \textbf{Right:} \nuPI is configured to recover \algo{Polyak}$(\beta=0.3$). We observe an increased eagerness to increase the multipliers \textit{as progress toward feasibility occurs}. This increases the chances of overshoot and subsequent oscillations.  
    \captioncomment{The blue, yellow, and red regions correspond to cases in which the updates performed by the \nuPI algorithm are faster, slower, or in the opposite direction than those of \GA, respectively. This plot illustrates the case $\xi_{t-1} > 0$. The middle and right figures presented here are the same as those in \cref{fig:nupi_one_over_x}. We include them here for the reader's convenience.}}
    \label{fig:nupi_one_over_x_appx}
\end{figure}

\textbf{PI (\cref{fig:nupi_one_over_x_appx}, middle).} We consider \nuPI$(\nu=0, \kp=1, \ki=1)$, which recovers \algo{PI}$(\kp=1, \ki=1)$. The update of the \algo{PI} optimizer relative to \GA~is as follows: 
\begin{enumerate}
    \item {\color[RGB]{180,200,254} \textbf{Mode A}} When either $e_t \geq \xi_{t-1}$ or $e_t\leq 0$, the relative update exceeds one and thus the \algo{PI} controller update can be seen as \textit{eager} compared to gradient ascent. 
    \begin{enumerate}
        \item When $e_t \geq \xi_{t-1}$, the constraint has historically been infeasible and the current violation indicates an \textit{increase in infeasibility}. In this case, \algo{PI} not only increases the value of the multiplier but does so more strongly than \GA. This proactive behavior serves to counteract the infeasibility increase. 
        \item When $e_t\leq 0$, the constraint has been satisfied despite historical infeasibility ($\xi_{t-1} > 0$). Here, the \algo{PI} controller decreases the multiplier much more than \GA. This serves to prevent overshoot into the feasible region.
    \end{enumerate}

    \item {\color[RGB]{234,215,80} \textbf{Mode B}} In the range $0 < \psi\xi_{t-1} \le e_t\le \xi_{t-1}$, the constraint at step $t$ (i) is not satisfied, (ii) it is smaller than the historical EMA of violations $\xi_{t-1}$, but not significantly (not beyond a factor of $\psi$). In this case, the \algo{PI} controller proactively exerts \textit{friction} by having a smaller update than \GA. This reduces the risk of overshoot under the assumption that the primal variables continue to make progress toward feasibility. 
    
    \item {\color[RGB]{255,180,180} \textbf{Mode C}} In the \textit{optimistic} phase, where $0 \le e_t \le \kappa \xi_{t-1}$, the \algo{PI} controller's update goes in the opposite direction to that of \GA: $\frac{\Delta \text{\nuPI}}{\Delta \text{GA}} \le 0$.
    This corresponds to a scenario where the constraint made significant progress toward feasibility relative to the historic violation EMA. 
    While \GA~would increase the multiplier in this case (since $g_t >0$), \algo{PI} \textit{decreases} the value of the multiplier. This is useful to prevent overshoot since significant progress toward feasibility is an indicator that the multiplier is already exerting sufficient pressure for the constraints to be satisfied.  
\end{enumerate}

\textbf{Negative momentum (\cref{fig:nupi_one_over_x_appx}, left).} We consider \Polyak~($\beta= -0.4$) as a realization of \nuPI, following \cref{thm:um_as_nupi}. We observe similar behavior to that of \nuPI$(\nu=0, \kp=1, \ki=1)$, in the middle figure of \cref{fig:nupi_one_over_x_appx}. Note that the current illustration assumes an equal value of the ``optimizer state'' $\xi_{t-1}$ between the momentum and non-momentum cases. However, the value of $\xi_{t}$ will be different depending on the momentum coefficient as $\beta = \nu$ also influences the update of $\xi$ (see \cref{algo:nupi}).

\textbf{Positive momentum (\cref{fig:nupi_one_over_x_appx}, right).} The right plot of \cref{fig:nupi_one_over_x_appx} considers \algo{Polyak} ($\beta= 0.3$) as a realization of \nuPI, following \cref{thm:um_as_nupi}. We observe significantly different behavior compared to the left and middle plots.
\begin{enumerate}
    \item {\color[RGB]{180,200,254} \textbf{Mode A}} When infeasibility is reduced, the algorithm is \textit{eager} to increase the multiplier more than \GA. This is a counter-intuitive operation of the algorithm considering that the current value of the multiplier can apply sufficient pressure to improve the constraint satisfaction. Increasing the multiplier further can lead to a higher risk of overshoot.
    
    \item {\color[RGB]{234,215,80} \textbf{Mode B}} Consider the cases in which infeasibility increases ($e_t \ge \xi_{t-1}$), or the constraints suddenly become (sufficiently) strictly feasible $e_t \le \psi \xi_{t-1} \le 0$.
    These cases induce \textit{frictioned} updates with the same sign as \GA, but of smaller magnitude.

    \item {\color[RGB]{255,180,180} \textbf{Mode C}} When the primal player is feasible, positive momentum would result in an \textit{increase} of the multiplier; going against the update of \GA, which would \textit{decrease} the multiplier. In this context, increasing the multiplier is unreasonable since the current value of the multiplier is already sufficient to achieve feasibility.
\end{enumerate}

\begin{figure}[h]
    \centering
    \includegraphics[scale=1.1]{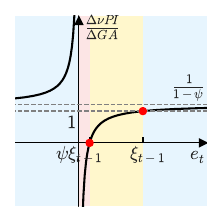}
    \includegraphics[scale=1.1]{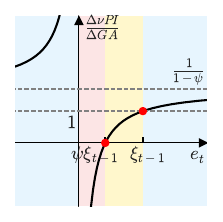}
    \includegraphics[scale=1.1]{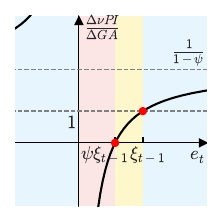}
    \caption[Effect of $\kp$ in the update of \nuPI relative to \GA.]{Effect of $\kp$ in the update of \nuPI relative to \GA. When $\kp$ approaches 0, \nuPI recovers \GA~(a constant function $y=1$ for the relative update). \textbf{A larger $\kp$ leads to a wider ``optimistic" region (in red) where \nuPI decreases the multiplier to prevent overshooting despite the constraint being violated.} \captioncomment{We use $\ki = 1$ and $\nu$ = 0 and $\kp$ of 0.2, 0.7 and 1.3, respectively.}}
    \label{fig:nupi_one_over_x_kp_appx}
\end{figure}

\textbf{Ablation on the influence of $\kp$}. \Cref{fig:nupi_one_over_x_kp_appx} presents three configurations of $\kp$ for a \nuPI$(\nu=0,\kp,\ki=1)$ optimizer. We display $\kp$ at 0.2, 0.7 and 1.3, respectively.
As $\kp \rightarrow 0$, \nuPI is equivalent to \GA. This is confirmed by the relative updates between \nuPI and \GA~converging to a constant function $y=1$.
As $\kp$ increases in the middle and right plots of \cref{fig:nupi_one_over_x_kp_appx}, the asymptote at $1/(1-\psi)$ moves further away from 1, and the width of the ``optimistic" region ({\color[RGB]{255,180,180} \textbf{Mode C}}) increases. 
In other words, as $\kp$ grows, the threshold for ``sufficient improvement'' is relaxed and the optimizer is more prone to decrease the multipliers upon improvements in constraint violation. 
This leads to a more ``cautious" behavior from the algorithm: the multiplier is decreased earlier when the problem approaches the feasible region, which prevents overshooting but with potentially slower convergence. 
One can monotonically control for the convergence and overshoot behaviors by adjusting the $\kp$ value, see \cref{fig:main_ablation} in \S \ref{sec:sparsity}.

\section{Analysis of continuous-time $\nu$PI dynamics as an oscillator}
\label{app:oscillator}

In this section, we examine the spectral properties of the gradient-descent/\nuPI flow dynamics presented in \cref{algo:continous_gd_nupi}. We extend the analysis of \citet{stooke2020responsive} (which only considers the dynamics of $\vx$) to also consider $\vlambda$. 

Consider a constrained optimization problem with equality constraints $\vh$. 
The GD/$\nu$PI flow corresponds to a \textit{continuous-time} dynamical system in which the primal player implements gradient descent on the Lagrangian, and the dual player implements $\nu$PI ascent. This is formalized in \cref{algo:continous_gd_nupi}.

\begin{algorithm}[h]
   \caption{Continuous-time gradient descent/\nuPI}
   \label{algo:continous_gd_nupi}
    \begin{algorithmic}
       \STATE {\bfseries Args:} proportional ($\kp$) and integral ($\ki$) gains for \nuPI flow
       \STATE \graytext{\small 1:} $\dot{\vx} = -\nabla f(\vx)- \Jac \vh(\vx) \vmu$
       \STATE \graytext{\small 2:} $\dot{\vmu} = \ki \vh(\vx) + \kp \dot{\vh}(\vx)$
    \end{algorithmic}
\end{algorithm}

\Cref{thm:oscillator_flow} characterizes the GD/$\nu$PI flow in \cref{algo:continous_gd_nupi} in terms of a second-order dynamical system. Note that this relationship holds for any constrained problem where the objective and constraints have second derivatives.
\Cref{app:qp} analyzes the resulting dynamical system for a quadratic program with linear equality constraints.

\subsection{Oscillator dynamics of GD/$\nu$PI flow}

\begin{theorem}
\label{thm:oscillator_flow}

The dynamics of \cref{algo:continous_gd_nupi} can be characterized by the following system of second-order differential equations, with initial conditions $\vx(0) = \vx_0$, $\vmu(0) = \vmu_0$, $\dot{\vx}(0) = -\nabla f(\vx_0)- \Jac \vh(\vx_0) \vmu_0$, and $\dot{\vmu}(0) = \ki \vh(\vx_0) + \Jac \vh(\vx_0) \dot{\vx}(0)$:
\begin{subnumcases}{}
    & $\ddot{\vx} = \displaystyle - \underbrace{\left(\nabla^2f + \sum_{c'} \mu_{c'} \nabla^2 \vh_{c'} \right)}_{\mPhi} \dot{\vx} - \Jac\vh \dot{\vmu}$ \label{eq:ode2_x} \\
   & $\ddot{\vmu} = \ki \Jac\vh\T \dot{\vx} + \kp \Jac\vh\T \ddot{\vx} + \kp \mXi$ \label{eq:ode2_lambda}
\end{subnumcases}
where $\mXi = \lspar \dot{\vx}\T \nabla^2 h_1 \dot{\vx}, \, \ldots, \, \dot{\vx}\T \nabla^2 h_c \dot{\vx} \rspar\T \in \mathbb{R}^c$.

This can be concisely represented in matrix form as:
\begin{align}
    \begin{bmatrix} \eye_{n\times n} & \vzero_{n\times c} \\  - \kp \Jac\vh\T & \eye_{c\times c}\end{bmatrix}  
    \begin{bmatrix} \ddot{\vx} \\ \ddot{\vmu} \end{bmatrix} 
    + 
    \begin{bmatrix} \mPhi & \Jac\vh \\  -\ki \Jac\vh\T & \vzero_{c\times c}\end{bmatrix}  
    \begin{bmatrix} \dot{\vx} \\ \dot{\vmu} \end{bmatrix} 
    + 
    \begin{bmatrix} \vzero  \\ - \beta \mXi \end{bmatrix} 
    = \vzero.
\end{align}

Or, equivalently:
\begin{align}
     \begin{bmatrix} \ddot{\vx} \\ \ddot{\vmu} \end{bmatrix} 
     + 
     \begin{bmatrix} \mPhi & \Jac\vh \\  \Jac\vh\T \lpar \kp \mPhi - \ki \eye \rpar & \kp \Jac\vh\T \Jac\vh \end{bmatrix}
     \begin{bmatrix} \dot{\vx} \\ \dot{\vmu} \end{bmatrix} 
     + 
     \begin{bmatrix} \vzero  \\ - \beta \mXi \end{bmatrix} 
     = \vzero.
\end{align}
    
\end{theorem}

\begin{proof}[\textbf{Proof of \cref{thm:oscillator_flow}}]

We start by computing the time derivatives of the objective gradient and constraint Jacobian: 
\begin{align}
\frac{d}{d t} \left[\nabla f\right] = \left[\sum_j \frac{\partial (\nabla f)}{\partial x_j} \frac{d x_j}{d t} \right]_{i}= \nabla^2f\dot{\vx} \hspace{9mm}
\frac{d}{d t} \left[\Jac \vh \right]  = \begin{bmatrix} \nabla^2 \vh_1 \dot{\vx} & \nabla^2 \vh_2 \dot{\vx} & \ldots & \nabla^2 \vh_C \dot{\vx} \end{bmatrix} \label{eq:jacobian_g}
\end{align}
Therefore, the second order dynamics for $\vx$ are given by:
\begin{subequations}
\begin{align}
    \ddot{\vx} &= \frac{d}{d t} \lspar -\nabla f(\vx) - \Jac \vhx \vmu \rspar = -\frac{d}{d t} \left[\nabla f\right] -\Jac \vh \dot{\vmu} - \frac{d}{d t} \left[\Jac \vh \right] \vmu 
    \label{eq:x_ddot_aux}
    \\
    &= -\nabla^2f\dot{\vx} -\Jac \vh \dot{\vmu} - \sum_{c'}\mu_{c'}\nabla^2 \vh_{c'} \dot{\vx} \\
    &=  - \underbrace{\left(\nabla^2f + \sum_{c'}\mu_{c'}\nabla^2 \vh_{c'} \right)}_{\mPhi} \dot{\vx} - \Jac \vh \dot{\vmu} \label{eq:pi_ode_x}
\end{align}
\end{subequations}
The second order dynamics for $\vmu$ are given by:
\begin{align}
    \ddot{\vmu} &= \frac{d}{d t} \left[\ki \vh + \kp \Jac \vh\T \dot{\vx} \right] = \alpha \dot{\vh} + \kp \mXi \dot{\vx} + \kp \Jac \vh\T \ddot{\vx},
\end{align}
where $\mXi$ is defined as:
\begin{align}
\mXi \defas \frac{d}{d t} \left[\Jac \vh\T \right]\dot{\vx} & = \begin{bmatrix} \dot{\vx}\T \nabla^2 \vh_1\dot{\vx}  & \dot{\vx}\T\nabla^2 \vh_2\dot{\vx}   & \ldots & \dot{\vx}\nabla^2 \vh_c\dot{\vx}   \end{bmatrix}\T .
\end{align}
\end{proof}

\vspace{-2ex}
\subsection{Dynamics of GD/$\nu$PI flow for a constrained quadratic program}
\label{app:qp}

Let $\mH \in \mathbb{R}^{n \times n}$ be positive semi-definite and consider the convex quadratic program with $c$ linear constraints:
\begin{equation}
    \label{eq:qp_problem}
    \min_{\vx} \frac{1}{2}\vx\T\mH\vx + \vc\T\vx \hspace{3mm} \text{subject to}  \hspace{3mm} \mA\vx - \vb = 0.
\end{equation}
The Lagrangian min-max game associated with the problem in  \cref{eq:qp_problem} is given by:
\begin{equation}
    \Lag(\vx, \vmu) = \frac{1}{2}\vx\T\mH\vx + \vc\T\vx + \vmu\T (\mA \vx - \vb) = \frac{1}{2}\vx\T\mH\vx + \vc\T\vx + \vmu\T \mA \vx - \vmu\T \vb.
\end{equation}
The linearity of the constraints in \cref{eq:qp_problem} implies $\Jac \vh = \mA\T$ and $\nabla^2 g_{c'} = \vzero$ for $c' \in [c]$, thus $\mPhi = \mH$ and $\mXi = \vzero$. Therefore, we obtain a homogeneous system of second-order differential equations with constant coefficients:
\begin{align}
     \begin{bmatrix} \ddot{\vx} \\ \ddot{\vmu} \end{bmatrix} 
     + 
     \underbrace{\begin{bmatrix} \mH & \mA\T \\  \mA \lpar \kp \mH - \ki \eye \rpar & \kp \mA \mA\T \end{bmatrix}}_{\mU}
     \begin{bmatrix} \dot{\vx} \\ \dot{\vmu} \end{bmatrix} 
     = \vzero.
\end{align}
A simple state transformation $\vz = \lspar \vx, \vmu, \dot{\vx}, \dot{\vmu} \rspar\T$ and $\dot{\vz} = \lspar \dot{\vx}, \dot{\vmu}, \ddot{\vx}, \ddot{\vmu} \rspar\T$ yields:
\begin{align}
     \dot{\vz} = - \begin{bmatrix} \vzero & \eye \\  
     \vzero & \mU
     \end{bmatrix} \vz = - \begin{bmatrix} \vzero_{(n+c)\times(n+c)} & \eye_{(n+c)\times(n+c)} \\  
     \vzero_{(n+c)\times(n+c)} & 
     \begin{bmatrix} \mH & \mA\T \\  \mA \lpar \kp \mH - \ki \eye \rpar & \kp \mA \mA\T \end{bmatrix}
     \end{bmatrix} \vz
\end{align}
Therefore, this $2(n+c)$-dimensional linear system has zero as an eigenvalue with algebraic multiplicity $n+c$, and the remaining eigenvalues correspond to the spectrum of $- \mU$.

When the matrix $\mH$ is zero, we recover the smooth bilinear games considered by \citet[Eq. 18]{gidel2019negative} in their study of negative momentum. In this case, the matrix $\mU$ looks like:
\begin{align}
    - \mU =  - \begin{bmatrix} \vzero & \mA\T \\ -\ki\mA & \kp \mA \mA\T \end{bmatrix}
\end{align}
It is easy to see that large enough values of $\kp$ cause the eigenvalues of the matrix to have negative real parts, and thus make the system converge. However, if $\kp = 0$ (i.e. gradient descent-ascent), the eigenvalues of this matrix are either 0 or pure imaginary. This fact is in line with existing results in the literature on the lack of convergence  gradient descent-ascent on bilinear games \cite{gidel2019negative}.

\textbf{Case of one-variable and one constraint.}
It is instructive to  analyze the spectrum of $\mU$ in the case of a problem with a one-dimensional primal variable and a single constraint (and thus one multiplier). In this case, $\mU$ and its eigenvalues take the form:
\begin{align}
    - \mU = - \begin{bmatrix} h & a \\  a \lpar \kp h - \ki \rpar & \kp a^2 \end{bmatrix} \hspace{10mm} \lambda = \frac{-\lpar h + \kp a^2 \rpar \pm \sqrt{\lpar h + \kp a^2 \rpar^2 - 4a^2\ki}}{2} 
    \label{eq:critical_damping}
\end{align}
As before, the eigenvalues of this matrix depend on the choice of $\kp$. This is illustrated in \cref{fig:2d_eigenvalues}.

\begin{figure}[h]
    \centering    \includegraphics[width=0.7\textwidth]{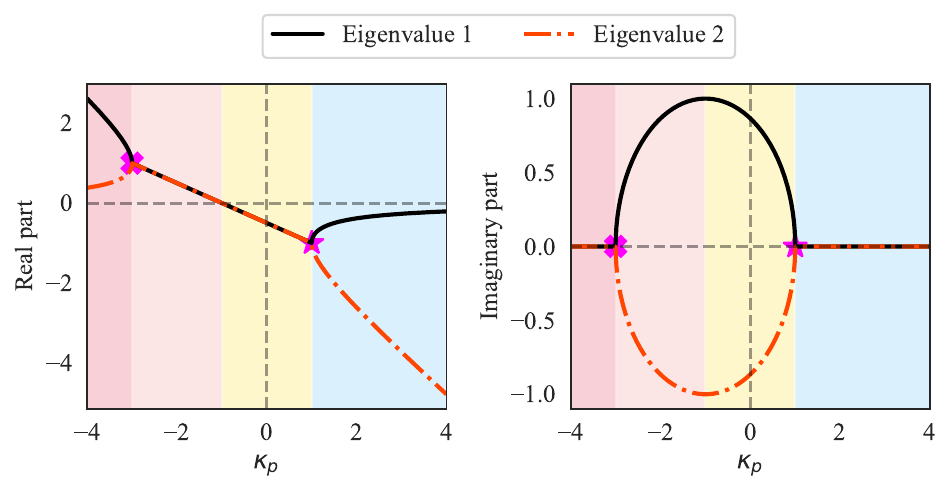}
    \caption[Eigenvalues for \cref{eq:qp_problem} as a function of $\kp$ in the one-dimensional case.]{ Eigenvalues for \cref{eq:qp_problem} as a function of $\kp$ in the one-dimensional case. \textbf{A positive value of $\kp$ (denoted by { $\star$}) achieves \textit{critical damping} (i.e. equal convergence rate for both dimensions).} \captioncomment{This plot uses $h=1$, $a=-1$ and $\ki=1$.}}
    \label{fig:2d_eigenvalues}
\end{figure}

Note that when the discriminant of \cref{eq:critical_damping} is zero, both eigenvalues match (and must thus be real). When this occurs and both eigenvalues are negative, the system converges and does so at the same rate in both dimensions. This is akin to the notion of \textit{critical damping} from the control theory literature.

The values of $\kp$ that make the discriminant zero are $\kp^* = \frac{- h \pm 2|a| \sqrt{\ki}}{a^2}$, leading to the eigenvalues $\lambda(\kp^*) = \frac{-\lpar h + \kp^* a^2 \rpar}{2} = \mp a\sqrt{\ki}$. These values of $\kp^*$ are marked with $\star$ and $\times$ in \cref{fig:2d_eigenvalues}. Note that out of the two values of $\kp$ producing matching eigenvalues, only the choice $\kp^* > 0 $ yields a convergent system.

More generally, depending on $\kp$, the system exhibits different behaviors:
\begin{itemize}
    \item \textbf{Divergence.} In the red regions, the system \textit{diverges}; in light red region, this happens together with oscillations. Note how all the divergent configurations use a negative value of $\kp$. The fuchsia cross ($\mathbf{\times}$) denotes the value of $\kp$ for which both dimensions diverge at the same rate.
    \item \textbf{Underdamping.} In the yellow region, the system is \textit{underdamped} and \textit{converges with oscillations}. Interestingly, this system admits some negative values of $\kp$ (of sufficiently small magnitude) while remaining convergent.
    \item \textbf{Critical damping.} The fuchsia star ($\star$) shows the $\kp$ value that makes both dimensions of the system converge \textit{at the same rate}. Note that this \textit{critical damping} regime is achieved at a strictly positive value of $\kp$, and thus is not achievable by gradient ascent.
    \item \textbf{Overdamping.} In the blue region, the system is \textit{convergent without oscillations} but \textit{overdamped} since the dimension corresponding to the black eigenvalue converges more slowly. 
\end{itemize}

\section{Illustrative 2D nonconvex problem} 

We demonstrate the behavior of \nuPI on the two-dimensional, nonconvex, equality-constrained problem in \cref{eq:2d_problem}. 
This problem was proposed by \citet{Boyd_2021}. 
The setting is simple enough to allow for visualizing the optimization paths of each optimization variable and multiplier, while also being challenging due to nonconvexity. 
\begin{equation}
    \label{eq:2d_problem}
    \underset{\vx = (x_1, x_2)}{\text{min}} \, f(\vx) \defas \left\Vert \begin{bmatrix}
           x_1 + e^{-x_2}\\
           x_1^2 + 2x_2 + 1
         \end{bmatrix} \right\Vert_2^2 
    \hspace{2mm}
    \text{subject to}
    \hspace{2mm}  
    h(\vx) \defas x_1 + x_1^3 + x_2 + x_2^2 = 2.
\end{equation}

\begin{figure*}[h]
    \centering
    \includegraphics{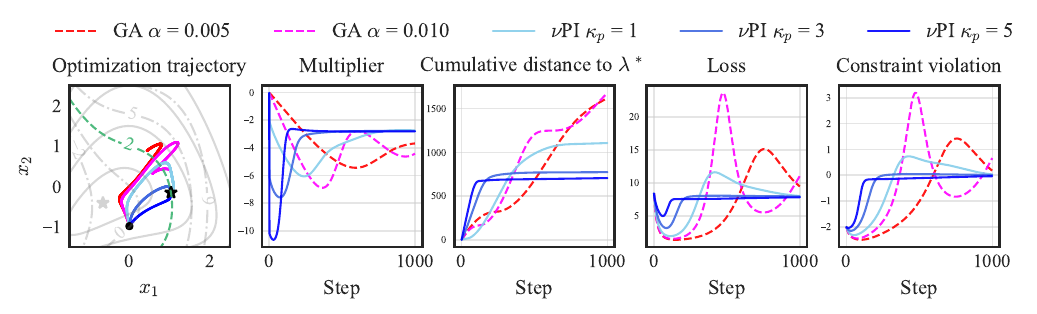}
    \vspace{-4ex}
    \caption[Optimization trajectories for different algorithms on a 2D nonconvex equality-constrained minimization problem.]{Optimization trajectories for different algorithms on a 2D nonconvex equality-constrained minimization problem. \captioncomment{\nuPI runs use $\nu = 0$ and $\ki = 0.01$. The light gray $\star$ marks the \textit{unconstrained} optimum, while the black $\star$ marks the \textit{constrained} optimum. Level sets correspond to the objective function (solid) and constraint (dashed).}}
    \label{fig:2d_kp_ablation}
\end{figure*}

\textbf{\GA~trajectories.} In \cref{fig:2d_kp_ablation}, \GA~trajectories are initially drawn toward the direction of the unconstrained optimum since multipliers grow slowly at first.
As training progresses, the constraint plays a more significant role. 
With a step-size that is too small ($\alpha = 0.005$), the trajectory does not converge to the global optimum. 
In contrast, the system reaches the global constrained optimum point when employing a larger step-size ($\alpha = 0.01$).
This is achieved while incurring in \textit{oscillations}. 
The phase change from not converging with a small step-size, to converging with oscillations indicates that \GA~is not suitable for obtaining critical damping when solving the problem. 

\textbf{\nuPI~trajectories.} The three blue trajectories in \cref{fig:2d_kp_ablation} show different behaviors of \nuPI: underdamping (light blue, $\kp=1$), almost-critical damping ($\kp=3$) and overdamping (dark blue, $\kp=5$). 
Note the monotonic effect of $\kp$ on the damping of the system.
\nuPI provides the flexibility to obtain different levels of constraint overshoot, and can achieve feasibility and convergence at different speeds. 
This added flexibility leads to enhanced control over the dynamics of the system relative to \GA, thus enabling applications of \nuPI to safety-sensitive tasks.

\begin{figure*}[h]
    \centering
    \includegraphics{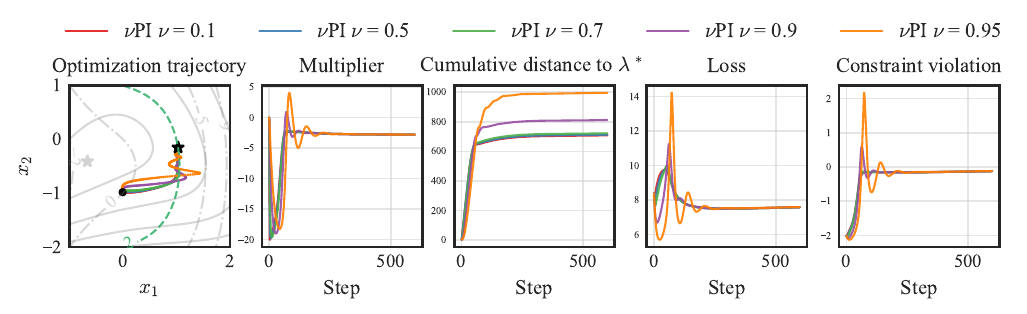}
    \vspace{-4ex}
    \caption[Optimization trajectories for the \nuPI algorithm under different choices of $\nu$.]{Optimization trajectories for the \nuPI algorithm under different choices of $\nu$. \captioncomment{\nuPI runs use $\ki = 0.01$ and $\kp = 10$.}
    }
    \label{fig:2d_nu_ablation}
\end{figure*}

\textbf{Ablation on $\nu$.} In \cref{fig:2d_nu_ablation}, we zoom in on the effect of $\nu$ for fixed choices of $\kp$ and $\ki$.
A $\nu$ closer to 0 tends towards \algo{PI}, whereas a $\nu$ closer to one gives more importance to historical constraint violations.
We observe that a larger $\nu$ behaves qualitatively similar to positive momentum: the multiplier tends to increase faster if the constraint is not satisfied for a period of time. 
In this example, this leads to oscillations as shown for $\nu=0.95$. 
Since this problem is deterministic, using a non-zero $\nu$ does not show any advantage. 
Our fairness experiments showcase an application where $\nu>0$ is beneficial.

\section{Experimental details}
\label{app:exp_details}

Our implementation use PyTorch 2.0.0 \citep{pytorch} and the Cooper library for Lagrangian constrained optimization \citep{gallegoPosada2022cooper}.
Our code is available at: \href{https://github.com/motahareh-sohrabi/nuPI}{\texttt{https://github.com/motahareh-sohrabi/nuPI}}.

\subsection{Linear SVM experiments}
\label{app:svm}

In our experiment with linear SVMs, we focus on two linearly separable classes from the Iris dataset \cite{misc_iris_53}. We select 100 instances from the Iris setosa and Iris versicolor species, which are two linearly separable classes. Each data point in this dataset has four features. We selected 70\% of data for training and the rest for validation. This gives the algorithm 70 Lagrange multipliers to learn. 

We know that a unique $\lambda^*$ exists in our experiments. The linearly independent constraint qualification (LICQ) holds for the selected data, so the Karush-Kuhn-Tucker (KKT) conditions imply the existence and uniqueness of $\lambda^*$ at the constrained optimum $x^*$. All of the methods that do not diverge achieve perfect training and validation accuracy in this task.

\textbf{Experiment configuration and hyperparameters.} Throughout all of the experiments, we fixed the primal optimizer and only changed the dual optimizer. The primal optimizer is gradient descent with momentum, with a step size of $10^{-3}$ and momentum of $0.9$. 

\begin{figure*}[h]
    \centering
    \includegraphics[scale=1.0]{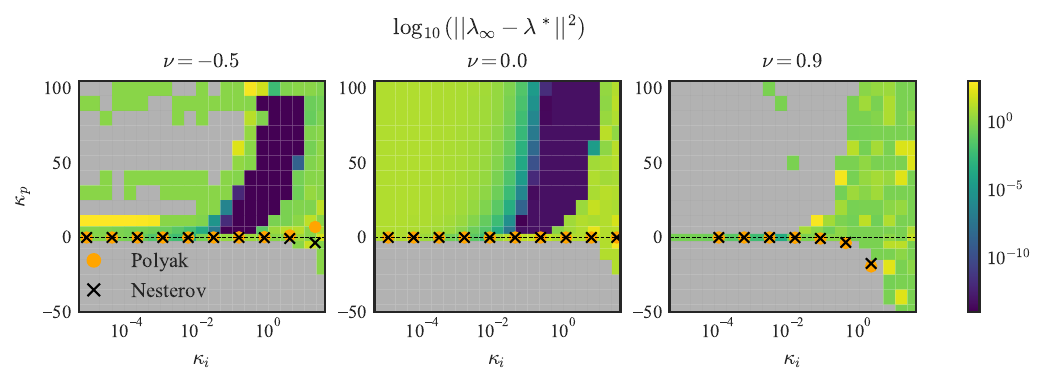}
    \caption[Distance to optimal Lagrange multipliers for different selections of parameter $\nu$ in $\nu$PI algorithm.]{Distance to optimal Lagrange multipliers for different selections of parameter $\nu$ in $\nu$PI algorithm. We also show where the equivalent $\kp$ and $\ki$ parameters for \algo{Nesterov}$(\alpha, \beta=\nu)$ and \algo{Polyak}$(\alpha, \beta=\nu)$ lie according to \cref{thm:um_as_nupi} for different values of $\alpha$. \textbf{The \nuPI algorithm with $\nu = 0$ can give the highest number of converging step-sizes. While negative $\nu = 0.5$ induces a range of converging step-sizes as well, there is no value of $\ki$ that the algorithm converges for $\nu = 0.9$.}\captioncomment{ The gray color shows the runs exceeding a distance of $10^3$ to $\lambda^*$.}}
    \label{fig:svm_heatmap_varying_u}
\end{figure*}

\textbf{Different values of the parameter $\nu$ in \nuPI algorithm.}
We examine how changing the parameter $\nu$ in the \cref{algo:nupi} can affect the convergence of the SVM task with different choices of $\ki$ and $\kp$. \cref{fig:svm_heatmap_varying_u} shows how \nuPI behaves when choosing a negative, zero and positive value of $\nu$. While $\nu = -0.5$ can lead to a converging algorithm for some step-sizes, $\nu = 0.0$ offers a wider range of converging step-sizes. There is no choice of step-size for which the \nuPI algorithm with a positive value of $\nu = 0.9$ converges to $\lambda^*$.

\textbf{Relationship between momentum and \nuPI algorithms.}
\cref{thm:um_as_nupi} suggests that Polyak and Nesterov momentum algorithms can be instantiated using a specific choice of $\ki$ and $\kp$ in the \cref{algo:nupi}. In \cref{fig:svm_heatmap_varying_u} we show where \algo{Polyak}$(\alpha, \beta)$ and \algo{Nesterov}$(\alpha, \beta)$ lie for a choice of $\alpha$'s and with $\beta = \nu$. For each pair of ($\alpha_i$, $\beta$), we calculate the value of $\ki$ and $\kp$ that recover \algo{Polyak}$(\alpha_i, \beta)$ and \algo{Nesterov}$(\alpha_i, \beta)$ according to \cref{thm:um_as_nupi}.
\begin{itemize}
    \item When $\beta = \nu = 0$ there is no momentum and both \Polyak~and \Nesterov~reduce to gradient ascent. Therefore, all of the dots indicating momentum methods lie in the $\kp = 0$ line for the middle plot of \cref{fig:svm_heatmap_varying_u}.
    \item A positive $\beta = \nu = 0.9$ (\cref{fig:svm_heatmap_varying_u}, right) corresponds to a common momentum choice in minimization problems. We can see how there is no step-size value for which \Polyak~or \Nesterov~converge in this task. This observation supports the claim of  \citet{gidel2019negative} on the ineffectiveness of positive momentum for convergence in games.
    \item A negative $\beta = \nu = -0.5$ leads to convergence for some choices of step-size. However, these do not correspond to what \Polyak~and \Nesterov~can achieve. This is consistent with our observation in \cref{fig:svm_sensitivity_analysis} (left), indicating the necessity of adding a (positive) $\kp$ term to the optimizer in order to achieve convergence. This highlights the benefits of the increased generality of \nuPI.

Moreover, we observe that \Polyak~with negative momentum achieves a positive $\kp$ while all momentum choices for \Nesterov~lead to negative $\kp$ values. This further supports the experimental results of \citet{gidel2019negative}, where \Polyak~is used when they want to experiment with negative momentum. Our hypothesis is that negative momentum with Polyak is successful in games because it can induce a positive $\kp$.
\end{itemize}

\subsection{Sparsity experiments}
\label{app:details:sparsity}

\textbf{Background.} \citet{louizos2017learning} propose a re-parameterization of models that allows applying $L_0$-norm regularization on their weights. They propose the use of stochastic gates $\vz$ that indicate whether each parameter is active or not, where $\vz$ follows a hard-concrete distribution parameterized by $\vphi$. Employing the re-parameterization trick allows the computation of gradients of the $L_0$-norm of the model with respect to $\vphi$. \citet{gallego2022controlled} formulate a constrained optimization problem that prescribes the desired sparsity of the model as a constraint.
\begin{equation}
    \label{eq:sparsity_cmp_app}
    \min_{\vw, \vphi \in \mathbb{R}^d} \mathbb{E}_{\vz | \vphi} \left[ L(\vw \odot \vz \,|\, \mathcal{D}) \right] \quad \text{s.t.} \quad \frac{\mathbb{E}_{\vz | \vphi} [\|\vz\|_0]}{\#(\vw)} \leq \epsilon,
\end{equation}
where $\vx$ are the parameters of the model, $L$ is an ERM objective, and $\mathcal{D}$ is a dataset. The constraint is normalized with the total number of parameters of the model $\#(\cdot)$, so that the constraint level $\epsilon \in [0, 1]$ corresponds to the maximum allowed \textit{proportion} of active parameters.
For details on the re-parameterization, and a closed form expression for $\mathbb{E}_{\vz | \vphi} [\|\vz\|_0]$, see \citet{louizos2017learning,gallego2022controlled}.

\textbf{Hard-concrete distribution.} The $L_0$-norm re-parameterization proposed by \citet{louizos2017learning} considers a hard-concrete distribution for the stochastic gates of the model. The hard-concrete distribution is based on a stretched and clamped concrete distribution \citep{maddison2016concrete}. Similar to \citet{louizos2017learning,gallego2022controlled}, we choose a temperature of $2/3$ for the concrete distribution, and a stretching interval of $[-0.1, 1.1]$.

\textbf{Architecture.} We consider ResNet-18 \citep{he2016deep} models with basic residual blocks for our sparsity experiments, which have a total of approximately $11.2$ million parameters. 
Following \citet{louizos2017learning} and \citet{gallego2022controlled}, we employ output feature map sparsity on the first convolutional layer of each residual block, whereas the following convolutional layer and the residual connection
are kept to be fully dense. The first convolutional layer of the model and the linear output layer are also kept fully dense. This model counts with 8 sparsifiable convolutional layers.

\textbf{Choice of sparsity levels.} Although \citet{gallego2022controlled} consider up to $80\%$ structured sparsity ($20\%$ density) for ResNet-18 models, \citet{gale2019state} indicate that it is possible to train ResNet-50 models with structured sparsity of up to $95\%$ ($5\%$ density or less), without incurring on a catastrophic loss on model accuracy. Therefore, we consider sparsity levels of between $30\%$ and $90\%$ ($70\%$ and $10\%$ density, respectively). 

\textbf{Primal optimization.} We consider an optimization pipeline for the model that incorporates standard techniques used to train $L_0$-sparse ResNet-18 models on CIFAR10. For the weights of the model, we use SGD with a momentum of $0.9$, an initial learning rate of $0.01$, and a cosine annealing learning rate scheduler \citep{loshchilov2016sgdr}. 

We initialize the gates with a \textit{droprate init} of $0.01$, effectively yielding a fully dense model at initialization. 
Akin to \citet{gallego2022controlled}, we use \Adam~\citep{kingma2014adam} with a step-size of $8 \cdot 10^{-4}$ to optimize the $\vphi$ parameters of the stochastic gates. 
When applying $L_2$-norm regularization on the parameters, we detach the contribution of the gates as recommended by \citet{gallego2022controlled}.

\textbf{Dual optimization.} For sparsity experiments, we consider $\nu = 0$. Since the constraint is deterministic given the state of the model (it does not need to be estimated from mini-batches), we consider the use of an EMA to not be crucial for this task.
Unless otherwise stated, we use a dual step-size of $8 \cdot 10^{-4}$ for all dual optimizer choices (as was provided by \citet{gallego2022controlled}). We decide against tuning the dual step-size separately for each optimizer to highlight the flexibility of \nuPI: given a step-size that was tuned to yield good results for \GA, \nuPI may produce better-behaved dynamics. 

All of our sparsity experiments use a batch size of 128 and are over 200 epochs.

\subsection{Fairness experiments}

\textbf{Dataset.} In this experiment we consider the adult dataset \cite{misc_adult_2}, pre-processed following \citet{zafar2019fairness}. The raw data comprises eight categorical and six continuous attributes. After processing, the data is comprised of 50-dimensional sparse feature vectors. The train and test sets consist of 30,162 and 15,060 samples, respectively.

\textbf{Background.} We consider a fairness task under the disparate impact constraint \cite{zafar2019fairness} shown in \cref{eq:fairness_problem}. 
This constraint is also known as statistical parity and demographic parity \cite{10.1145/3097983.3098095,10.1145/2090236.2090255}. We consider two sensitive attributes in the adult dataset: sex, denoted as $A_1 = \{\textit{male},\,  \textit{female}\}$, and race, denoted as $A_2 = \{\textit{White},\,\textit{Black},\, \textit{Asian-Pac-Islander},\, \textit{Amer-Indian-Eskimo}, \,\textit{Other}\}$. \Cref{eq:fairness_problem} entails the intersection of both attributes, leading to $|A_1| \times |A_2| = 10$ constraints.

\textbf{Architecture and primal optimization.} We train a 2-hidden-layer neural network with hidden sizes of (100, 100) similar to the experimental setup of \citet{cotter2019proxy}.
In order to choose the primal optimizer hyperparamters, we trained the unconstrained problem and chose the parameters of the run with the highest training accuracy. We fixed this primal optimizer across our constrained experiments to be \Adam~($\alpha = 10^{-2}$). 

\textbf{Dual optimization.} We chose  the best step-size for \GA~aimed at minimizing training accuracy, while ensuring that the maximum violation achieves the lowest possible value. This lead to a dual step-size of $\alpha = 0.03$. We then fixed this value as the $\ki$ parameter of \nuPI and ran a grid search to find the best $\kp$. The grid search for $\kp$ considered (logarithmically spaced) values in $[0.01, 100] $. The best results were found with $\kp=5$.

Due to the noise in the constraints, we also experimented with the effect of $\nu$ on the optimization dynamics. We tried $\nu$ values of 0.0, 0.5, 0.9, 0.95, and 0.99. We noticed that higher values of $\nu$ can improve the learning dynamics, with the best results achieved at 0.99. Setting $\nu = 0$ results in noisy Lagrange multipliers, which lead to unstable optimization. This is illustrated in \cref{fig:fairness_plot}.

\section{Comprehensive results on the sparsity task}

In this section we provide extensive experiment results for our sparsity experiments, complementing \S\ref{sec:sparsity}. We conducted experiments with global and layer-wise sparsity targets, at $\epsilon = 70\%,\, 50\%,\, 30\%,\, 10\%$ density levels. The shaded region of our plots corresponds to the feasible set. ``Relative violations'' are computed by dividing the absolute constraint violations by the target density.

\subsection{Global}
\label{app:experiments:sparsity}

For global sparsity experiments (\cref{fig:icml2024:icml_global_70,fig:icml2024:icml_global_50,fig:icml2024:icml_global_30,fig:icml2024:icml_global_10,table*:global_density_70_icml,table*:global_density_50_icml,table*:global_density_30_icml,table*:global_density_10_icml}), we observe a general trend for models that overshoot into becoming excessively sparse to achieve a lower training performance. This insight is also generally true for validation accuracy. In particular, gradient ascent and momentum methods consistently exhibit overshoot, whereas \nuPI and gradient ascent with dual restarts do not overshoot and achieve good final performance. Dual restarts generally produce (slightly) infeasible solutions.
Note that at $\epsilon = 10\%$, negative momentum runs do not overshoot, but positive momentum runs do.

\begin{figure}[h!]
    \centering
    \includegraphics[scale=1]{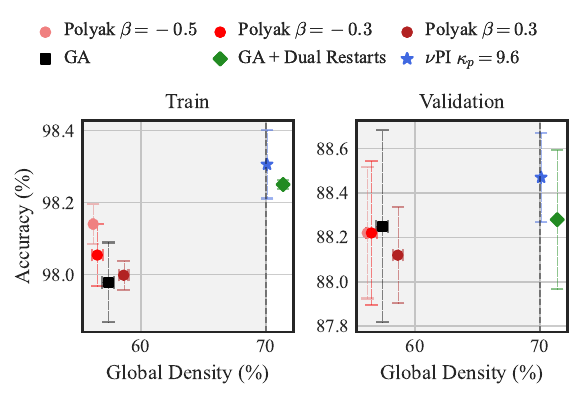}
    \caption[CIFAR10 trade-off plot for \textit{global} sparsity under a 70\% density target.]{CIFAR10 trade-off plot for \textit{global} sparsity under a 70\% density target. \textbf{\nuPI successfully achieves the desired sparsity while achieving the highest train accuracy.}
    \captioncomment{The shaded region is the feasible set. As higher density correlates to higher train accuracy, overshooting to a lower density is undesirable. All optimizers use the same step-size. \textit{This figure is the same as \cref{fig:main_tradeoff}. We repeat it here for the reader's convenience.}}
    }
    \label{fig:icml2024:icml_global_70}
\end{figure}

\begin{table*}[h!]
\centering
\caption[CIFAR10 results for \textit{global} sparsity under a $70\%$ density target.]{CIFAR10 results for \textit{global} sparsity under a $70\%$ density target. \textbf{\nuPI successfully achieves the desired sparsity while achieving the highest train accuracy.} \captioncomment{The results in this table correspond to those in \cref{fig:icml2024:icml_global_70}. As higher density correlates to higher train accuracy, overshooting to a lower density is undesirable. All optimizers use the same step-size.}}
\label{table*:global_density_70_icml}
\begin{adjustbox}{max width=0.8\textwidth}
\begin{tabular}{l|cc|cc}
\toprule
Method & Train Acc. & Test Acc. & Violation & Relative Violation \\
\midrule
Polyak $\beta = -0.5$ & 98.1\color{gray}{\small{ $\pm$ 0.06}} & 88.2\color{gray}{\small{ $\pm$ 0.30}} & -13.8\color{gray}{\small{ $\pm$ 0.09}} & -19.7\color{gray}{\small{ $\pm$ 0.13}} \\
Polyak $\beta = -0.3$ & 98.1\color{gray}{\small{ $\pm$ 0.09}} & 88.2\color{gray}{\small{ $\pm$ 0.32}} & -13.5\color{gray}{\small{ $\pm$ 0.43}} & -19.3\color{gray}{\small{ $\pm$ 0.61}} \\
Polyak $\beta = 0.3$ & 98.0\color{gray}{\small{ $\pm$ 0.04}} & 88.1\color{gray}{\small{ $\pm$ 0.22}} & -11.4\color{gray}{\small{ $\pm$ 0.39}} & -16.3\color{gray}{\small{ $\pm$ 0.55}} \\
GA & 98.0\color{gray}{\small{ $\pm$ 0.11}} & 88.2\color{gray}{\small{ $\pm$ 0.43}} & -12.6\color{gray}{\small{ $\pm$ 0.48}} & -18.0\color{gray}{\small{ $\pm$ 0.69}} \\
GA + Dual Restarts & 98.2\color{gray}{\small{ $\pm$ 0.01}} & 88.3\color{gray}{\small{ $\pm$ 0.31}} & 1.4\color{gray}{\small{ $\pm$ 0.07}} & 2.0\color{gray}{\small{ $\pm$ 0.10}} \\
\textit{Ours} - $\nu$PI $\kappa_p = 9.6$ & 98.3\color{gray}{\small{ $\pm$ 0.09}} & 88.5\color{gray}{\small{ $\pm$ 0.20}} & 0.1\color{gray}{\small{ $\pm$ 0.01}} & 0.1\color{gray}{\small{ $\pm$ 0.02}} \\
\bottomrule
\end{tabular}
\end{adjustbox}
\end{table*}

\begin{figure}[h!]
    \centering
    \includegraphics[scale=1]{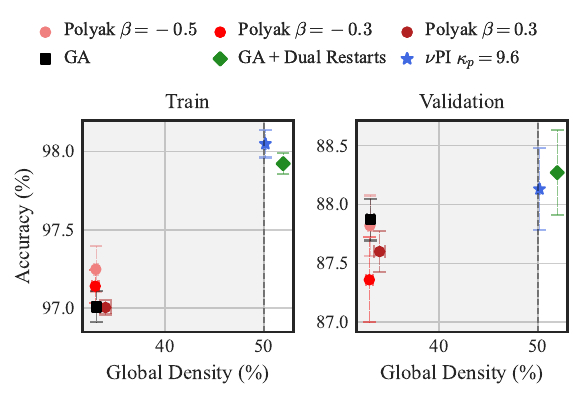}
    \caption{CIFAR10 trade-off plot for \textit{global} sparsity under a 50\% density target.}
    \label{fig:icml2024:icml_global_50}
\end{figure}

\begin{table*}[h!]
\centering
\caption[CIFAR10 results for \textit{global} sparsity under a $50\%$ density target.]{CIFAR10 results for \textit{global} sparsity under a $50\%$ density target. \captioncomment{The results in this table correspond to those in \cref{fig:icml2024:icml_global_50}.}}
\label{table*:global_density_50_icml}
\begin{adjustbox}{max width=0.8\textwidth}
\begin{tabular}{l|cc|cc}
\toprule
Method & Train Acc. & Test Acc. & Violation & Relative Violation \\
\midrule
Polyak $\beta = -0.5$ & 97.2\color{gray}{\small{ $\pm$ 0.15}} & 87.8\color{gray}{\small{ $\pm$ 0.26}} & -17.0\color{gray}{\small{ $\pm$ 0.17}} & -34.0\color{gray}{\small{ $\pm$ 0.34}} \\
Polyak $\beta = -0.3$ & 97.1\color{gray}{\small{ $\pm$ 0.11}} & 87.4\color{gray}{\small{ $\pm$ 0.36}} & -17.1\color{gray}{\small{ $\pm$ 0.33}} & -34.2\color{gray}{\small{ $\pm$ 0.66}} \\
Polyak $\beta = 0.3$ & 97.0\color{gray}{\small{ $\pm$ 0.04}} & 87.6\color{gray}{\small{ $\pm$ 0.17}} & -16.0\color{gray}{\small{ $\pm$ 0.59}} & -32.1\color{gray}{\small{ $\pm$ 1.18}} \\
GA & 97.0\color{gray}{\small{ $\pm$ 0.10}} & 87.9\color{gray}{\small{ $\pm$ 0.18}} & -17.0\color{gray}{\small{ $\pm$ 0.36}} & -33.9\color{gray}{\small{ $\pm$ 0.72}} \\
GA + Dual Restarts & 97.9\color{gray}{\small{ $\pm$ 0.07}} & 88.3\color{gray}{\small{ $\pm$ 0.36}} & 2.0\color{gray}{\small{ $\pm$ 0.09}} & 3.9\color{gray}{\small{ $\pm$ 0.18}} \\
\textit{Ours} - $\nu$PI $\kappa_p = 9.6$ & 98.0\color{gray}{\small{ $\pm$ 0.09}} & 88.1\color{gray}{\small{ $\pm$ 0.35}} & 0.2\color{gray}{\small{ $\pm$ 0.03}} & 0.3\color{gray}{\small{ $\pm$ 0.05}} \\
\bottomrule
\end{tabular}
\end{adjustbox}
\end{table*}

\begin{figure}[h!]
    \centering
    \includegraphics[scale=1]{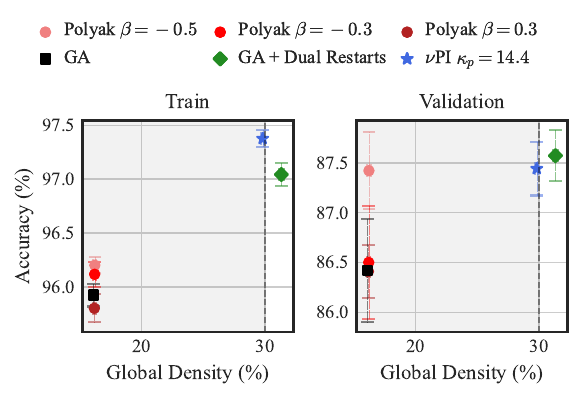}
    \caption{CIFAR10 trade-off plot for \textit{global} sparsity under a 30\% density target.}
    \label{fig:icml2024:icml_global_30}
\end{figure}

\begin{table*}[h!]
\centering
\caption[CIFAR10 results for \textit{global} sparsity under a $30\%$ density target.]{CIFAR10 results for \textit{global} sparsity under a $30\%$ density target. \captioncomment{The results in this table correspond to those in \cref{fig:icml2024:icml_global_30}.}}
\label{table*:global_density_30_icml}
\begin{adjustbox}{max width=0.8\textwidth}
\begin{tabular}{l|cc|cc}
\toprule
Method & Train Acc. & Test Acc. & Violation & Relative Violation \\
\midrule
Polyak $\beta = -0.5$ & 96.2\color{gray}{\small{ $\pm$ 0.07}} & 87.4\color{gray}{\small{ $\pm$ 0.39}} & -13.8\color{gray}{\small{ $\pm$ 0.13}} & -45.9\color{gray}{\small{ $\pm$ 0.43}} \\
Polyak $\beta = -0.3$ & 96.1\color{gray}{\small{ $\pm$ 0.11}} & 86.5\color{gray}{\small{ $\pm$ 0.57}} & -13.8\color{gray}{\small{ $\pm$ 0.11}} & -45.9\color{gray}{\small{ $\pm$ 0.36}} \\
Polyak $\beta = 0.3$ & 95.8\color{gray}{\small{ $\pm$ 0.13}} & 86.4\color{gray}{\small{ $\pm$ 0.26}} & -13.8\color{gray}{\small{ $\pm$ 0.09}} & -46.0\color{gray}{\small{ $\pm$ 0.31}} \\
GA & 95.9\color{gray}{\small{ $\pm$ 0.11}} & 86.4\color{gray}{\small{ $\pm$ 0.52}} & -13.9\color{gray}{\small{ $\pm$ 0.11}} & -46.3\color{gray}{\small{ $\pm$ 0.36}} \\
GA + Dual Restarts & 97.0\color{gray}{\small{ $\pm$ 0.11}} & 87.6\color{gray}{\small{ $\pm$ 0.26}} & 1.3\color{gray}{\small{ $\pm$ 0.22}} & 4.4\color{gray}{\small{ $\pm$ 0.73}} \\
\textit{Ours} - $\nu$PI $\kappa_p = 14.4$ & 97.4\color{gray}{\small{ $\pm$ 0.08}} & 87.4\color{gray}{\small{ $\pm$ 0.27}} & -0.2\color{gray}{\small{ $\pm$ 0.11}} & -0.7\color{gray}{\small{ $\pm$ 0.38}} \\
\bottomrule
\end{tabular}
\end{adjustbox}
\end{table*}

\begin{figure}[h!]
    \centering
    \includegraphics[scale=1]{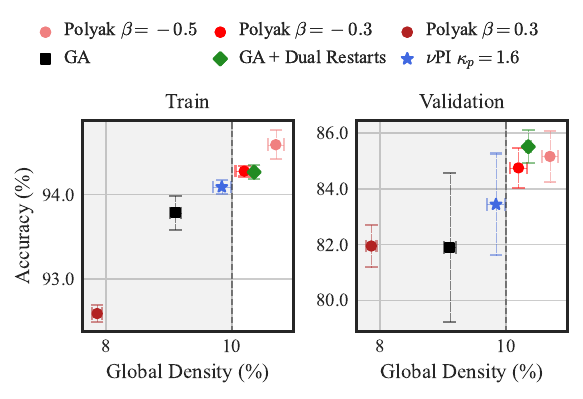}
    \caption{CIFAR10 trade-off plot for \textit{global} sparsity under a 10\% density target.}
    \label{fig:icml2024:icml_global_10}
\end{figure}

\begin{table*}[h!]
\centering
\caption[CIFAR10 results for \textit{global} sparsity under a $10\%$ density target.]{CIFAR10 results for \textit{global} sparsity under a $10\%$ density target. \captioncomment{The results in this table correspond to those in \cref{fig:icml2024:icml_global_10}.}}
\label{table*:global_density_10_icml}
\begin{adjustbox}{max width=0.8\textwidth}
\begin{tabular}{l|cc|cc}
\toprule
Method & Train Acc. & Test Acc. & Violation & Relative Violation \\
\midrule
Polyak $\beta = -0.5$ & 94.6\color{gray}{\small{ $\pm$ 0.17}} & 85.2\color{gray}{\small{ $\pm$ 0.93}} & 0.7\color{gray}{\small{ $\pm$ 0.13}} & 7.0\color{gray}{\small{ $\pm$ 1.31}} \\
Polyak $\beta = -0.3$ & 94.3\color{gray}{\small{ $\pm$ 0.06}} & 84.7\color{gray}{\small{ $\pm$ 0.71}} & 0.2\color{gray}{\small{ $\pm$ 0.14}} & 2.0\color{gray}{\small{ $\pm$ 1.39}} \\
Polyak $\beta = 0.3$ & 92.6\color{gray}{\small{ $\pm$ 0.10}} & 81.9\color{gray}{\small{ $\pm$ 0.77}} & -2.1\color{gray}{\small{ $\pm$ 0.08}} & -21.4\color{gray}{\small{ $\pm$ 0.82}} \\
GA & 93.8\color{gray}{\small{ $\pm$ 0.20}} & 81.9\color{gray}{\small{ $\pm$ 2.68}} & -0.9\color{gray}{\small{ $\pm$ 0.09}} & -9.0\color{gray}{\small{ $\pm$ 0.93}} \\
GA + Dual Restarts & 94.3\color{gray}{\small{ $\pm$ 0.08}} & 85.5\color{gray}{\small{ $\pm$ 0.59}} & 0.4\color{gray}{\small{ $\pm$ 0.03}} & 3.5\color{gray}{\small{ $\pm$ 0.35}} \\
\textit{Ours} - $\nu$PI $\kappa_p = 1.6$ & 94.1\color{gray}{\small{ $\pm$ 0.08}} & 83.4\color{gray}{\small{ $\pm$ 1.83}} & -0.2\color{gray}{\small{ $\pm$ 0.14}} & -1.6\color{gray}{\small{ $\pm$ 1.44}} \\
\bottomrule
\end{tabular}
\end{adjustbox}
\end{table*}

\subsection{Layer-wise}

We perform layer-wise sparsity experiments with $\epsilon=10\%,30\%,50\%,70\%$ density targets (\cref{table*:layer-wise_density_70_icml,table*:layer-wise_density_50_icml,table*:layer-wise_density_30_icml,table*:layer-wise_density_10_icml}). 
We observe a similar trend to global sparsity experiments: \GA~and momentum methods overshoot, while \nuPI and \GA~with dual restarts reliably achieve feasible solutions, with small levels of overshoot. Moreover, we observe that the violations of \nuPI and \GA~with dual restarts span a small range of values relative to other methods. 
This highlights the robustness of \nuPI since the $\kp$ coefficient did not need to be tuned independently for each constraint.

\begin{table*}[h!]
\centering
\caption[CIFAR10 results for \textit{layer-wise} sparsity under a $70\%$ density target.]{CIFAR10 results for \textit{layer-wise} sparsity under a $70\%$ density target. GA and momentum methods overshoot to \textit{different} values for each constraint. 
\textbf{\nuPI achieves the desired sparsity \textit{on all layers} while achieving the highest train accuracy.} \captioncomment{As higher density correlates to higher train accuracy, overshooting to a lower density is undesirable. All optimizers use the same step-size. \textit{This table is the same as \cref{table*:layerwise_sparsity}. We repeat it here for the reader's convenience.} }}
\label{table*:layer-wise_density_70_icml}
\hspace*{-7mm}
\begin{adjustbox}{max width=1.\textwidth}
\begin{tabular}{l|rr|rrr|rrr}
\toprule
\multirow{2}{*}{Method} & \multicolumn{2}{|c}{Accuracy} & \multicolumn{3}{|c}{Violation} & \multicolumn{3}{|c}{Relative Violation} \\
 & \multicolumn{1}{|c}{Train} & \multicolumn{1}{c|}{Test} & \multicolumn{1}{|c}{Min} & \multicolumn{1}{c}{Max} & \multicolumn{1}{c}{Range} & \multicolumn{1}{|c}{Min} & \multicolumn{1}{c}{Max} & \multicolumn{1}{c}{Range} \\
\midrule
Polyak $\beta = -0.5$ & 91.9\color{gray}{\small{ $\pm$ 0.18}} & 83.6\color{gray}{\small{ $\pm$ 1.40}} & -26.5\color{gray}{\small{ $\pm$ 0.81}} & -7.9\color{gray}{\small{ $\pm$ 0.86}} & 18.9\color{gray}{\small{ $\pm$ 1.31}} & -37.8\color{gray}{\small{ $\pm$ 1.15}} & -11.4\color{gray}{\small{ $\pm$ 1.22}} & 27.0\color{gray}{\small{ $\pm$ 1.87}} \\
Polyak $\beta = -0.3$ & 92.1\color{gray}{\small{ $\pm$ 0.07}} & 83.4\color{gray}{\small{ $\pm$ 1.44}} & -27.1\color{gray}{\small{ $\pm$ 0.73}} & -6.7\color{gray}{\small{ $\pm$ 0.38}} & 20.6\color{gray}{\small{ $\pm$ 0.92}} & -38.8\color{gray}{\small{ $\pm$ 1.05}} & -9.6\color{gray}{\small{ $\pm$ 0.55}} & 29.4\color{gray}{\small{ $\pm$ 1.31}} \\
Polyak $\beta = 0.3$ & 91.9\color{gray}{\small{ $\pm$ 0.20}} & 82.5\color{gray}{\small{ $\pm$ 1.50}} & -26.3\color{gray}{\small{ $\pm$ 0.82}} & -2.3\color{gray}{\small{ $\pm$ 0.69}} & 24.0\color{gray}{\small{ $\pm$ 0.88}} & -37.5\color{gray}{\small{ $\pm$ 1.17}} & -3.2\color{gray}{\small{ $\pm$ 0.99}} & 34.3\color{gray}{\small{ $\pm$ 1.26}} \\
GA & 92.0\color{gray}{\small{ $\pm$ 0.08}} & 84.1\color{gray}{\small{ $\pm$ 1.97}} & -27.8\color{gray}{\small{ $\pm$ 0.49}} & -5.2\color{gray}{\small{ $\pm$ 0.39}} & 22.0\color{gray}{\small{ $\pm$ 0.56}} & -39.6\color{gray}{\small{ $\pm$ 0.70}} & -7.4\color{gray}{\small{ $\pm$ 0.55}} & 31.4\color{gray}{\small{ $\pm$ 0.80}} \\
GA + Dual Restarts & 95.0\color{gray}{\small{ $\pm$ 0.22}} & 85.3\color{gray}{\small{ $\pm$ 0.61}} & -0.0\color{gray}{\small{ $\pm$ 0.00}} & 1.2\color{gray}{\small{ $\pm$ 0.28}} & 1.2\color{gray}{\small{ $\pm$ 0.28}} & -0.0\color{gray}{\small{ $\pm$ 0.00}} & 1.8\color{gray}{\small{ $\pm$ 0.40}} & 1.8\color{gray}{\small{ $\pm$ 0.40}} \\
\textit{Ours} - $\nu$PI $\kappa_p = 8.0$ & 95.1\color{gray}{\small{ $\pm$ 0.06}} & 86.2\color{gray}{\small{ $\pm$ 0.46}} & -1.7\color{gray}{\small{ $\pm$ 0.27}} & 0.1\color{gray}{\small{ $\pm$ 0.04}} & 1.8\color{gray}{\small{ $\pm$ 0.29}} & -2.4\color{gray}{\small{ $\pm$ 0.38}} & 0.2\color{gray}{\small{ $\pm$ 0.06}} & 2.5\color{gray}{\small{ $\pm$ 0.42}} \\
\bottomrule
\end{tabular}
\end{adjustbox}
\end{table*}

\begin{table*}[h!]
\centering
\caption[CIFAR10 results for \textit{layer-wise} sparsity under a $50\%$ density target.]{CIFAR10 results for \textit{layer-wise} sparsity under a $50\%$ density target. \captioncomment{As higher density correlates to higher train accuracy, overshooting to a lower density is undesirable. All optimizers use the same step-size.}}
\label{table*:layer-wise_density_50_icml}
\hspace*{-7mm}
\begin{adjustbox}{max width=1.\textwidth}
\begin{tabular}{l|rr|rrr|rrr}
\toprule
\multirow{2}{*}{Method} & \multicolumn{2}{|c}{Accuracy} & \multicolumn{3}{|c}{Violation} & \multicolumn{3}{|c}{Relative Violation} \\
 & \multicolumn{1}{|c}{Train} & \multicolumn{1}{c|}{Test} & \multicolumn{1}{|c}{Min} & \multicolumn{1}{c}{Max} & \multicolumn{1}{c}{Range} & \multicolumn{1}{|c}{Min} & \multicolumn{1}{c}{Max} & \multicolumn{1}{c}{Range} \\
\midrule
Polyak $\beta = -0.5$ & 87.5\color{gray}{\small{ $\pm$ 0.17}} & 80.2\color{gray}{\small{ $\pm$ 2.65}} & -32.6\color{gray}{\small{ $\pm$ 0.95}} & -15.9\color{gray}{\small{ $\pm$ 0.80}} & 16.4\color{gray}{\small{ $\pm$ 1.31}} & -65.1\color{gray}{\small{ $\pm$ 1.89}} & -31.7\color{gray}{\small{ $\pm$ 1.59}} & 32.9\color{gray}{\small{ $\pm$ 2.61}} \\
Polyak $\beta = -0.3$ & 87.7\color{gray}{\small{ $\pm$ 0.21}} & 80.3\color{gray}{\small{ $\pm$ 2.13}} & -29.5\color{gray}{\small{ $\pm$ 0.69}} & -15.4\color{gray}{\small{ $\pm$ 0.81}} & 13.6\color{gray}{\small{ $\pm$ 1.23}} & -59.1\color{gray}{\small{ $\pm$ 1.38}} & -30.8\color{gray}{\small{ $\pm$ 1.62}} & 27.2\color{gray}{\small{ $\pm$ 2.46}} \\
Polyak $\beta = 0.3$ & 87.4\color{gray}{\small{ $\pm$ 0.21}} & 79.7\color{gray}{\small{ $\pm$ 3.18}} & -30.9\color{gray}{\small{ $\pm$ 0.66}} & -14.1\color{gray}{\small{ $\pm$ 0.25}} & 16.9\color{gray}{\small{ $\pm$ 0.70}} & -61.8\color{gray}{\small{ $\pm$ 1.33}} & -28.3\color{gray}{\small{ $\pm$ 0.49}} & 33.9\color{gray}{\small{ $\pm$ 1.40}} \\
GA & 87.6\color{gray}{\small{ $\pm$ 0.12}} & 77.5\color{gray}{\small{ $\pm$ 3.14}} & -29.7\color{gray}{\small{ $\pm$ 1.02}} & -14.2\color{gray}{\small{ $\pm$ 0.60}} & 14.7\color{gray}{\small{ $\pm$ 1.15}} & -59.4\color{gray}{\small{ $\pm$ 2.04}} & -28.3\color{gray}{\small{ $\pm$ 1.20}} & 29.4\color{gray}{\small{ $\pm$ 2.30}} \\
GA + Dual Restarts & 92.8\color{gray}{\small{ $\pm$ 0.07}} & 83.5\color{gray}{\small{ $\pm$ 0.81}} & -0.0\color{gray}{\small{ $\pm$ 0.01}} & 1.0\color{gray}{\small{ $\pm$ 0.46}} & 1.0\color{gray}{\small{ $\pm$ 0.46}} & -0.0\color{gray}{\small{ $\pm$ 0.02}} & 1.9\color{gray}{\small{ $\pm$ 0.92}} & 1.9\color{gray}{\small{ $\pm$ 0.93}} \\
\textit{Ours} - $\nu$PI $\kappa_p = 8.0$ & 93.2\color{gray}{\small{ $\pm$ 0.06}} & 83.6\color{gray}{\small{ $\pm$ 0.87}} & -1.5\color{gray}{\small{ $\pm$ 0.13}} & 0.1\color{gray}{\small{ $\pm$ 0.08}} & 1.6\color{gray}{\small{ $\pm$ 0.19}} & -2.9\color{gray}{\small{ $\pm$ 0.26}} & 0.2\color{gray}{\small{ $\pm$ 0.16}} & 3.2\color{gray}{\small{ $\pm$ 0.37}} \\
\bottomrule
\end{tabular}
\end{adjustbox}
\end{table*}

\begin{table*}[h!]
\centering
\caption[CIFAR10 results for \textit{layer-wise} sparsity under a $30\%$ density target.]{CIFAR10 results for \textit{layer-wise} sparsity under a $30\%$ density target. \captioncomment{As higher density correlates to higher train accuracy, overshooting to a lower density is undesirable. All optimizers use the same step-size.}}
\label{table*:layer-wise_density_30_icml}
\hspace*{-7mm}
\begin{adjustbox}{max width=1.\textwidth}
\begin{tabular}{l|rr|rrr|rrr}
\toprule
\multirow{2}{*}{Method} & \multicolumn{2}{|c}{Accuracy} & \multicolumn{3}{|c}{Violation} & \multicolumn{3}{|c}{Relative Violation} \\
 & \multicolumn{1}{|c}{Train} & \multicolumn{1}{c|}{Test} & \multicolumn{1}{|c}{Min} & \multicolumn{1}{c}{Max} & \multicolumn{1}{c}{Range} & \multicolumn{1}{|c}{Min} & \multicolumn{1}{c}{Max} & \multicolumn{1}{c}{Range} \\
\midrule
Polyak $\beta = -0.5$ & 81.8\color{gray}{\small{ $\pm$ 0.19}} & 63.5\color{gray}{\small{ $\pm$ 18.58}} & -25.2\color{gray}{\small{ $\pm$ 1.54}} & -17.0\color{gray}{\small{ $\pm$ 0.37}} & 8.4\color{gray}{\small{ $\pm$ 1.73}} & -84.1\color{gray}{\small{ $\pm$ 5.12}} & -56.8\color{gray}{\small{ $\pm$ 1.23}} & 28.0\color{gray}{\small{ $\pm$ 5.77}} \\
Polyak $\beta = -0.3$ & 82.1\color{gray}{\small{ $\pm$ 0.54}} & 63.3\color{gray}{\small{ $\pm$ 8.65}} & -25.1\color{gray}{\small{ $\pm$ 1.15}} & -16.4\color{gray}{\small{ $\pm$ 0.38}} & 8.7\color{gray}{\small{ $\pm$ 0.91}} & -83.5\color{gray}{\small{ $\pm$ 3.84}} & -54.7\color{gray}{\small{ $\pm$ 1.27}} & 29.0\color{gray}{\small{ $\pm$ 3.04}} \\
Polyak $\beta = 0.3$ & 81.8\color{gray}{\small{ $\pm$ 0.32}} & 72.7\color{gray}{\small{ $\pm$ 3.36}} & -25.1\color{gray}{\small{ $\pm$ 2.12}} & -17.5\color{gray}{\small{ $\pm$ 0.24}} & 7.4\color{gray}{\small{ $\pm$ 2.12}} & -83.6\color{gray}{\small{ $\pm$ 7.07}} & -58.5\color{gray}{\small{ $\pm$ 0.79}} & 24.8\color{gray}{\small{ $\pm$ 7.07}} \\
GA & 81.8\color{gray}{\small{ $\pm$ 0.44}} & 72.7\color{gray}{\small{ $\pm$ 4.40}} & -24.8\color{gray}{\small{ $\pm$ 1.11}} & -17.0\color{gray}{\small{ $\pm$ 0.60}} & 8.5\color{gray}{\small{ $\pm$ 1.22}} & -82.5\color{gray}{\small{ $\pm$ 3.69}} & -56.7\color{gray}{\small{ $\pm$ 1.99}} & 28.2\color{gray}{\small{ $\pm$ 4.07}} \\
GA + Dual Restarts & 89.7\color{gray}{\small{ $\pm$ 0.23}} & 82.9\color{gray}{\small{ $\pm$ 2.59}} & -0.0\color{gray}{\small{ $\pm$ 0.00}} & 0.9\color{gray}{\small{ $\pm$ 0.33}} & 0.9\color{gray}{\small{ $\pm$ 0.33}} & -0.0\color{gray}{\small{ $\pm$ 0.01}} & 3.0\color{gray}{\small{ $\pm$ 1.10}} & 3.0\color{gray}{\small{ $\pm$ 1.10}} \\
\textit{Ours} - $\nu$PI $\kappa_p = 12.0$ & 89.8\color{gray}{\small{ $\pm$ 0.11}} & 82.0\color{gray}{\small{ $\pm$ 2.45}} & -0.3\color{gray}{\small{ $\pm$ 0.13}} & 0.3\color{gray}{\small{ $\pm$ 0.03}} & 0.6\color{gray}{\small{ $\pm$ 0.12}} & -0.8\color{gray}{\small{ $\pm$ 0.42}} & 1.0\color{gray}{\small{ $\pm$ 0.11}} & 2.1\color{gray}{\small{ $\pm$ 0.39}} \\
\bottomrule
\end{tabular}
\end{adjustbox}
\end{table*}

\begin{table*}[h!]
\centering
\caption[CIFAR10 results for \textit{layer-wise} sparsity under a $10\%$ density target.]{CIFAR10 results for \textit{layer-wise} sparsity under a $10\%$ density target. As higher density correlates to higher train accuracy, overshooting to a lower density is undesirable.  \captioncomment{All optimizers use the same step-size.}}
\label{table*:layer-wise_density_10_icml}
\hspace*{-7mm}
\begin{adjustbox}{max width=1.\textwidth}
\begin{tabular}{l|rr|rrr|rrr}
\toprule
\multirow{2}{*}{Method} & \multicolumn{2}{|c}{Accuracy} & \multicolumn{3}{|c}{Violation} & \multicolumn{3}{|c}{Relative Violation} \\
 & \multicolumn{1}{|c}{Train} & \multicolumn{1}{c|}{Test} & \multicolumn{1}{|c}{Min} & \multicolumn{1}{c}{Max} & \multicolumn{1}{c}{Range} & \multicolumn{1}{|c}{Min} & \multicolumn{1}{c}{Max} & \multicolumn{1}{c}{Range} \\
\midrule
Polyak $\beta = -0.5$ & 71.3\color{gray}{\small{ $\pm$ 0.61}} & 61.0\color{gray}{\small{ $\pm$ 9.50}} & -10.0\color{gray}{\small{ $\pm$ 0.14}} & -5.9\color{gray}{\small{ $\pm$ 0.47}} & 4.0\color{gray}{\small{ $\pm$ 0.48}} & -100.0\color{gray}{\small{ $\pm$ 1.36}} & -58.7\color{gray}{\small{ $\pm$ 4.74}} & 40.5\color{gray}{\small{ $\pm$ 4.83}} \\
Polyak $\beta = -0.3$ & 70.9\color{gray}{\small{ $\pm$ 0.60}} & 49.5\color{gray}{\small{ $\pm$ 16.33}} & -10.0\color{gray}{\small{ $\pm$ 0.01}} & -5.9\color{gray}{\small{ $\pm$ 0.60}} & 4.1\color{gray}{\small{ $\pm$ 0.60}} & -100.0\color{gray}{\small{ $\pm$ 0.11}} & -58.9\color{gray}{\small{ $\pm$ 5.97}} & 41.1\color{gray}{\small{ $\pm$ 5.95}} \\
Polyak $\beta = 0.3$ & 69.2\color{gray}{\small{ $\pm$ 0.71}} & 56.3\color{gray}{\small{ $\pm$ 15.05}} & -10.0\color{gray}{\small{ $\pm$ 0.02}} & -6.7\color{gray}{\small{ $\pm$ 0.06}} & 3.3\color{gray}{\small{ $\pm$ 0.08}} & -100.0\color{gray}{\small{ $\pm$ 0.15}} & -67.3\color{gray}{\small{ $\pm$ 0.65}} & 32.7\color{gray}{\small{ $\pm$ 0.79}} \\
GA & 71.0\color{gray}{\small{ $\pm$ 0.32}} & 49.6\color{gray}{\small{ $\pm$ 11.1}} & -10.0\color{gray}{\small{ $\pm$ 0.19}} & -6.1\color{gray}{\small{ $\pm$ 0.25}} & 3.9\color{gray}{\small{ $\pm$ 0.42}} & -100.0\color{gray}{\small{ $\pm$ 1.91}} & -61.2\color{gray}{\small{ $\pm$ 2.54}} & 38.8\color{gray}{\small{ $\pm$ 4.24}} \\
GA + Dual Restarts & 83.1\color{gray}{\small{ $\pm$ 0.27}} & 73.1\color{gray}{\small{ $\pm$ 4.87}} & -0.0\color{gray}{\small{ $\pm$ 0.00}} & 1.6\color{gray}{\small{ $\pm$ 0.14}} & 1.6\color{gray}{\small{ $\pm$ 0.14}} & -0.0\color{gray}{\small{ $\pm$ 0.02}} & 16.1\color{gray}{\small{ $\pm$ 1.39}} & 16.1\color{gray}{\small{ $\pm$ 1.40}} \\
\textit{Ours} - $\nu$PI $\kappa_p = 12.0$ & 81.4\color{gray}{\small{ $\pm$ 0.39}} & 42.8\color{gray}{\small{ $\pm$ 14.54}} & -1.9\color{gray}{\small{ $\pm$ 0.34}} & 0.9\color{gray}{\small{ $\pm$ 0.46}} & 3.2\color{gray}{\small{ $\pm$ 0.72}} & -19.1\color{gray}{\small{ $\pm$ 3.41}} & 9.3\color{gray}{\small{ $\pm$ 4.62}} & 31.7\color{gray}{\small{ $\pm$ 7.17}} \\
\bottomrule
\end{tabular}
\end{adjustbox}
\end{table*}

\section{Additional Experiments}

In this section, we include additional experimental results on the sparsity-constrained task. We analyze the dynamics of the multiplier throughout training in \cref{app:dynamics}, and conduct ablation studies on $\kp$ for \nuPI, the momentum coefficients of \algo{Polyak} and \algo{Nesterov}, and the step-size of \algo{Adam}. 

\subsection{Dynamics}
\label{app:dynamics}

The dynamics shown in \cref{fig:icml2024:dynamics_comparison} illustrate the change of the constraint violation and multipliers throughout optimization. We observe that \GA, \Polyak, and \Adam~quickly lead to overshoot into the feasible region, leading to overly sparse models. As training progresses however, these methods move closer to the boundary of the feasible region, reversing the initial overshoot. This recovery is most notorious for \Adam, whose multiplier decreases quickly after feasibility. \GA~with dual restarts sets the value of the multiplier to zero as soon as feasibility is achieved, thus preventing an incursion into the feasible set. \nuPI produces well-behaved multipliers and successfully avoids constraint overshoot.

\begin{figure*}[h!]
    \centering
    \includegraphics[scale=0.8,trim={0 5mm 0 0},clip]{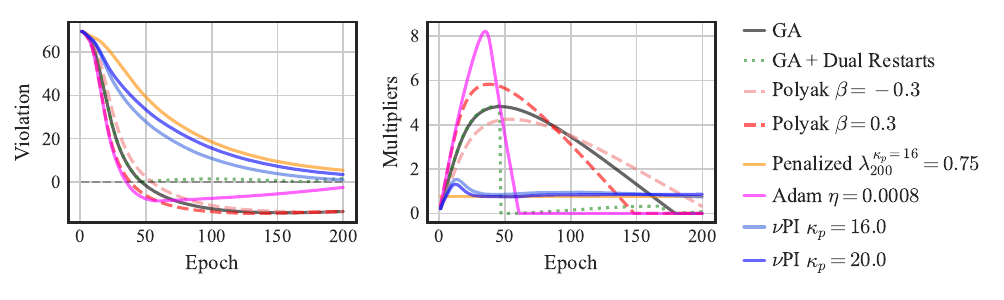}
    \caption{Dynamics plot for global sparsity under a 30\% density target.}
    \label{fig:icml2024:dynamics_comparison}
\end{figure*}

Note that for this sparsity task, it is reasonable to expect that the constraint is active at the constrained optimum since more model capacity correlates with better performance. However, note that \nuPI is the only method that provides a non-zero estimate of the Lagrange multiplier. The usefulness of Lagrange multiplier estimates is highlighted in \cref{fig:icml2024:dynamics_penalized}.

\begin{figure}[h!]
    \centering
    \includegraphics[scale=0.8,trim={0 0mm 0 0},clip]{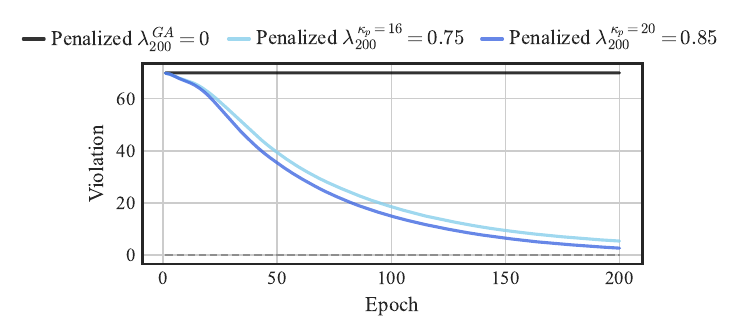}
    \vspace{-1ex}
    \caption{Dynamics plot for global sparsity under a 30\% density target.}
    \label{fig:icml2024:dynamics_penalized}
\end{figure}

\Cref{fig:icml2024:dynamics_penalized} considers unconstrained $L_0$-regularization experiments.
We use the final value of the multipliers corresponding to \nuPI$(\kp=16)$ and \nuPI$(\kp=20)$ runs as the (fixed) penalty coefficient for 200 epochs in the penalized formulation of the problem, akin to \citet{louizos2017learning}. We also include an experiment using the multiplier estimate from \GA (equal to zero).

\newpage
Unsurprisingly, the run with the multiplier of zero leads to 100\% density, since the sparsity penalty does not exert any influence during training. In contrast, the runs with the \nuPI multiplier estimates not only lead to sparse models but are also very close to the desired model density by the end of training.
This is remarkable since the problem we are solving is nonconvex, and optimal Lagrange multiplier values may not even exist.

\subsection{Ablation on the value of $\kp$}

In this section, we fix $\ki$ for \nuPI and ablate on the hyperparameter $\kp$ for two sparsity levels. The results are presented in  \cref{fig:icml2024:kp_ablation_global} and \cref{tab:icml2024:global_density_50_kp_ablation,tab:icml2024:global_density_30_kp_ablation}. 

\vspace{-1ex}

\begin{figure}[h!]
    \centering
    \includegraphics[scale=1]{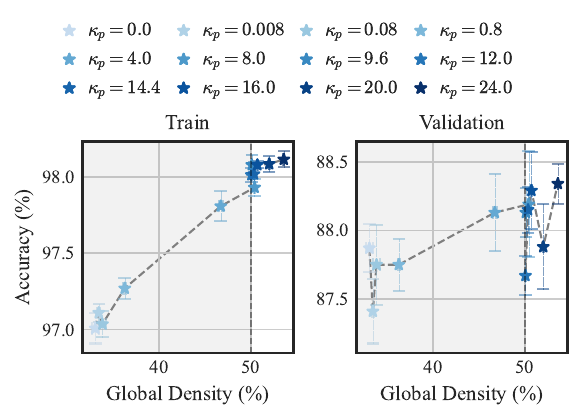}
    \includegraphics[scale=1,trim={0 0 0 18mm},clip]{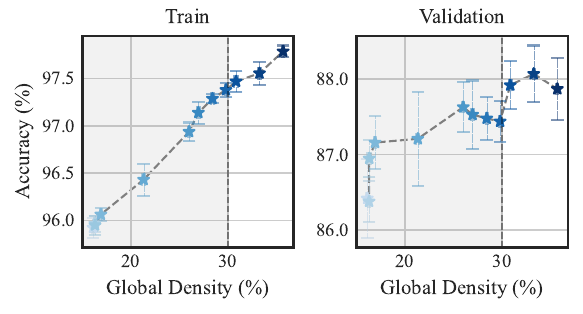}
    \vspace{-2ex}
    \caption{Ablation on trade-offs achievable by \nuPI under global density targets of 50\% (top) and 30\% (bottom).}
    \label{fig:icml2024:kp_ablation_global}
\end{figure}

We see that a larger $\kp$ leads to more damping and less overshoot. Note that there is a strong correlation between training accuracy and model density. Hence, it is important to be able to control overshoot in sparsity constraints and take advantage of the maximum allowed density for the sake of accuracy. There is a range of $\kp$ that can achieve such desired sparsity. The same trend roughly extends to validation accuracy (with some caveats due to generalization errors).

\begin{table*}[h!]
\centering
\caption{Ablation on the $\kp$ hyperparameter for a CIFAR10 task with a global density target of $\epsilon= 50\%$. \textbf{$\kp$ monotonically controls the degree of damping and constraint overshoot.}}
\label{tab:icml2024:global_density_50_kp_ablation}
\begin{adjustbox}{max width=0.8\textwidth}
\begin{tabular}{c|cc|cc}
\toprule
$\nu$PI  $\kappa_p$& Train Acc. & Test Acc. & Violation & Relative Violation \\
\midrule
0 & 97.0\color{gray}{\small{ $\pm$ 0.10}} & 87.9\color{gray}{\small{ $\pm$ 0.18}} & -17.0\color{gray}{\small{ $\pm$ 0.36}} & -33.9\color{gray}{\small{ $\pm$ 0.72}} \\
0.008 & 97.1\color{gray}{\small{ $\pm$ 0.06}} & 87.4\color{gray}{\small{ $\pm$ 0.23}} & -16.6\color{gray}{\small{ $\pm$ 0.14}} & -33.2\color{gray}{\small{ $\pm$ 0.28}} \\
0.08 & 97.0\color{gray}{\small{ $\pm$ 0.09}} & 87.7\color{gray}{\small{ $\pm$ 0.29}} & -16.2\color{gray}{\small{ $\pm$ 0.42}} & -32.4\color{gray}{\small{ $\pm$ 0.84}} \\
0.8 & 97.3\color{gray}{\small{ $\pm$ 0.07}} & 87.7\color{gray}{\small{ $\pm$ 0.19}} & -13.8\color{gray}{\small{ $\pm$ 0.13}} & -27.5\color{gray}{\small{ $\pm$ 0.27}} \\
4 & 97.8\color{gray}{\small{ $\pm$ 0.10}} & 88.1\color{gray}{\small{ $\pm$ 0.28}} & -3.3\color{gray}{\small{ $\pm$ 0.23}} & -6.5\color{gray}{\small{ $\pm$ 0.46}} \\
8 & 97.9\color{gray}{\small{ $\pm$ 0.06}} & 88.2\color{gray}{\small{ $\pm$ 0.39}} & 0.4\color{gray}{\small{ $\pm$ 0.03}} & 0.9\color{gray}{\small{ $\pm$ 0.06}} \\
9.6 & 98.1\color{gray}{\small{ $\pm$ 0.07}} & 88.1\color{gray}{\small{ $\pm$ 0.18}} & 0.1\color{gray}{\small{ $\pm$ 0.02}} & 0.2\color{gray}{\small{ $\pm$ 0.04}} \\
12 & 98.0\color{gray}{\small{ $\pm$ 0.05}} & 87.7\color{gray}{\small{ $\pm$ 0.14}} & 0.1\color{gray}{\small{ $\pm$ 0.01}} & 0.2\color{gray}{\small{ $\pm$ 0.03}} \\
14.4 & 98.0\color{gray}{\small{ $\pm$ 0.08}} & 88.2\color{gray}{\small{ $\pm$ 0.17}} & 0.4\color{gray}{\small{ $\pm$ 0.02}} & 0.7\color{gray}{\small{ $\pm$ 0.05}} \\
16 & 98.1\color{gray}{\small{ $\pm$ 0.02}} & 88.3\color{gray}{\small{ $\pm$ 0.28}} & 0.7\color{gray}{\small{ $\pm$ 0.02}} & 1.5\color{gray}{\small{ $\pm$ 0.04}} \\
20 & 98.1\color{gray}{\small{ $\pm$ 0.05}} & 87.9\color{gray}{\small{ $\pm$ 0.31}} & 2.0\color{gray}{\small{ $\pm$ 0.03}} & 4.0\color{gray}{\small{ $\pm$ 0.07}} \\
24 & 98.1\color{gray}{\small{ $\pm$ 0.05}} & 88.3\color{gray}{\small{ $\pm$ 0.15}} & 3.6\color{gray}{\small{ $\pm$ 0.03}} & 7.2\color{gray}{\small{ $\pm$ 0.06}} \\
\bottomrule
\end{tabular}
\end{adjustbox}
\end{table*}

\begin{table*}[h!]
\centering
\caption{Ablation on the $\kp$ hyperparameter for a CIFAR10 task with a global density target of $\epsilon= 30\%$.}
\label{tab:icml2024:global_density_30_kp_ablation}
\begin{adjustbox}{max width=0.8\textwidth}
\begin{tabular}{c|cc|cc}
\toprule
$\nu$PI  $\kappa_p$ & Train Acc. & Test Acc. & Violation & Relative Violation \\
\midrule
0 & 95.9\color{gray}{\small{ $\pm$ 0.11}} & 86.4\color{gray}{\small{ $\pm$ 0.52}} & -13.9\color{gray}{\small{ $\pm$ 0.11}} & -46.3\color{gray}{\small{ $\pm$ 0.36}} \\
0.008 & 96.0\color{gray}{\small{ $\pm$ 0.08}} & 86.4\color{gray}{\small{ $\pm$ 0.27}} & -13.7\color{gray}{\small{ $\pm$ 0.13}} & -45.7\color{gray}{\small{ $\pm$ 0.43}} \\
0.08 & 95.9\color{gray}{\small{ $\pm$ 0.10}} & 86.9\color{gray}{\small{ $\pm$ 0.25}} & -13.7\color{gray}{\small{ $\pm$ 0.18}} & -45.6\color{gray}{\small{ $\pm$ 0.59}} \\
0.8 & 96.1\color{gray}{\small{ $\pm$ 0.07}} & 87.2\color{gray}{\small{ $\pm$ 0.35}} & -13.1\color{gray}{\small{ $\pm$ 0.13}} & -43.6\color{gray}{\small{ $\pm$ 0.45}} \\
4 & 96.4\color{gray}{\small{ $\pm$ 0.17}} & 87.2\color{gray}{\small{ $\pm$ 0.62}} & -8.7\color{gray}{\small{ $\pm$ 0.16}} & -28.9\color{gray}{\small{ $\pm$ 0.54}} \\
8 & 96.9\color{gray}{\small{ $\pm$ 0.10}} & 87.6\color{gray}{\small{ $\pm$ 0.33}} & -4.0\color{gray}{\small{ $\pm$ 0.10}} & -13.3\color{gray}{\small{ $\pm$ 0.32}} \\
9.6 & 97.1\color{gray}{\small{ $\pm$ 0.12}} & 87.5\color{gray}{\small{ $\pm$ 0.45}} & -3.0\color{gray}{\small{ $\pm$ 0.18}} & -10.1\color{gray}{\small{ $\pm$ 0.60}} \\
12 & 97.3\color{gray}{\small{ $\pm$ 0.04}} & 87.5\color{gray}{\small{ $\pm$ 0.28}} & -1.6\color{gray}{\small{ $\pm$ 0.11}} & -5.3\color{gray}{\small{ $\pm$ 0.37}} \\
14.4 & 97.4\color{gray}{\small{ $\pm$ 0.08}} & 87.4\color{gray}{\small{ $\pm$ 0.27}} & -0.2\color{gray}{\small{ $\pm$ 0.11}} & -0.7\color{gray}{\small{ $\pm$ 0.38}} \\
16 & 97.5\color{gray}{\small{ $\pm$ 0.11}} & 87.9\color{gray}{\small{ $\pm$ 0.32}} & 0.8\color{gray}{\small{ $\pm$ 0.17}} & 2.8\color{gray}{\small{ $\pm$ 0.57}} \\
20 & 97.6\color{gray}{\small{ $\pm$ 0.12}} & 88.1\color{gray}{\small{ $\pm$ 0.38}} & 3.3\color{gray}{\small{ $\pm$ 0.11}} & 10.9\color{gray}{\small{ $\pm$ 0.36}} \\
24 & 97.8\color{gray}{\small{ $\pm$ 0.06}} & 87.9\color{gray}{\small{ $\pm$ 0.41}} & 5.7\color{gray}{\small{ $\pm$ 0.11}} & 19.0\color{gray}{\small{ $\pm$ 0.36}} \\
\bottomrule
\end{tabular}
\end{adjustbox}
\end{table*}

\subsection{\Adam}

We also experimented with a range of learning choices for \Adam~to explore their effect on constraint satisfaction and overshoot. The results are shown in \cref{fig:icml2024:adam_ablation_global}, and \cref{tab:icml2024:adam_ablation_global_50,tab:icml2024:adam_ablation_global_30}.

\begin{figure}[h!]
    \centering
    \includegraphics[scale=1]{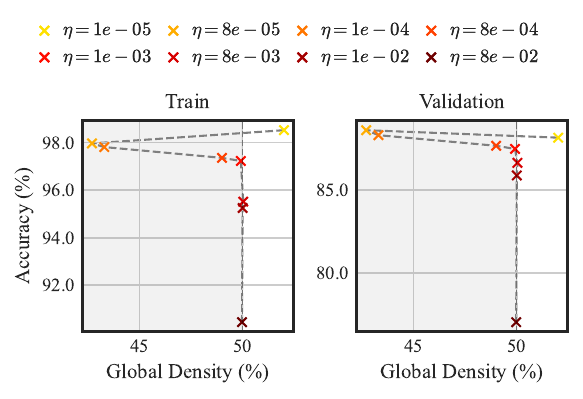}
    \includegraphics[scale=1,trim={0 0 0 15mm},clip]{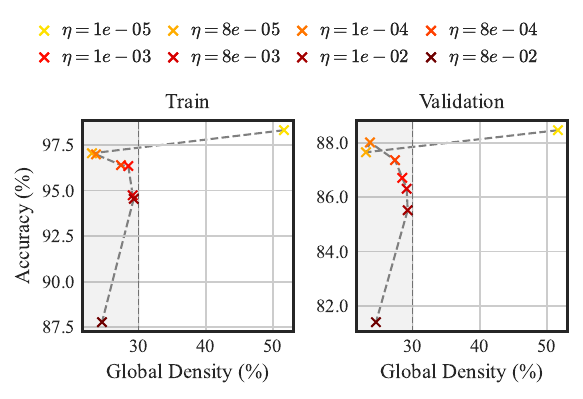}
    \caption{Ablation on the density-accuracy trade-offs achievable by \Adam~under global density targets of 50\% (top) and 30\% (bottom).}
    \label{fig:icml2024:adam_ablation_global}
\end{figure}

We observe that the influence of \Adam's learning on the constraint overshoot is not monotonic. When the step-size is small, \Adam~runs do not satisfy the constraint at the end of training. As the step-size increases, satisfaction is achieved together with varying degrees of overshoot into the feasible region. A range of larger step-sizes that lie at the sweet spot of almost exact constraint satisfaction. 

The sensitivity and non-monotonicity of the step-size make the tuning of the step-size hyperparameter for \Adam~challenging. 
Note that we restricted our experiments to the default EMA coefficients for \Adam~following PyTorch: $\beta_1=0.9$ and $\beta_2= 0.999$.

\begin{table*}[h!]
\centering
\caption{Ablation on the step-size hyperparameter for \Adam~on a CIFAR10 task with a global density target of $\epsilon= 50\%$.}
\label{tab:icml2024:adam_ablation_global_50}
\begin{adjustbox}{max width=1\textwidth}
\begin{tabular}{c|cc|cc}
\toprule
Adam $\eta$ & Train Acc. & Test Acc. & Violation \\
\midrule
$1\cdot 10^{-5}$ & 98.52 & 88.17 & 2.02 \\
$8\cdot 10^{-5}$ & 97.97 & 88.62 & -7.30 \\
$1\cdot 10^{-4}$ & 97.81 & 88.32 & -6.70 \\
$8\cdot 10^{-4}$ & 97.36 & 87.68 & -0.99 \\
$1\cdot 10^{-3}$ & 97.23 & 87.49 & -0.08 \\
$8\cdot 10^{-3}$ & 95.52 & 86.65 & 0.04 \\
$1\cdot 10^{-2}$ & 95.25 & 85.89 & 0.01 \\
$8\cdot 10^{-2}$ & 90.45 & 77.04 & -0.02 \\
\bottomrule
\end{tabular}
\end{adjustbox}
\end{table*}

\begin{table*}[h!]
\centering
\caption{Ablation on the step-size hyperparameter for \Adam~on a CIFAR10 task with a global density target of $\epsilon= 30\%$.}
\label{tab:icml2024:adam_ablation_global_30}
\begin{adjustbox}{max width=1\textwidth}
\begin{tabular}{c|cc|cc}
\toprule
Adam $\eta$ & Train Acc. & Test Acc. & Violation \\
\midrule
$1\cdot 10^{-5}$ & 98.32 & 88.46 & 21.66 \\
$8\cdot 10^{-5}$ & 97.05 & 87.65 & -6.94 \\
$1\cdot 10^{-4}$ & 96.99 & 88.01 & -6.34 \\
$8\cdot 10^{-4}$ & 96.40 & 87.36 & -2.61 \\
$1\cdot 10^{-3}$ & 96.35 & 86.71 & -1.53 \\
$8\cdot 10^{-3}$ & 94.74 & 86.31 & -0.88 \\
$1\cdot 10^{-2}$ & 94.55 & 85.52 & -0.73 \\
$8\cdot 10^{-2}$ & 87.78 & 81.41 & -5.45 \\
\bottomrule
\end{tabular}
\end{adjustbox}
\end{table*}

\subsection{Momentum}

We carried out similar ablations on the momentum coefficient of \Polyak~and \Nesterov using both positive and negative values. The results are shown in \cref{fig:icml2024:momentum_ablation_global_30}, and \cref{tab:icml2024:nesterov_ablation_global_30,tab:icml2024:polyak_ablation_global_30}.
We observe significant overshoot into the feasible region for all attempted values, compared to the desired target density of 30\%.

\begin{figure}[h!]
    \centering
    \includegraphics[scale=1]{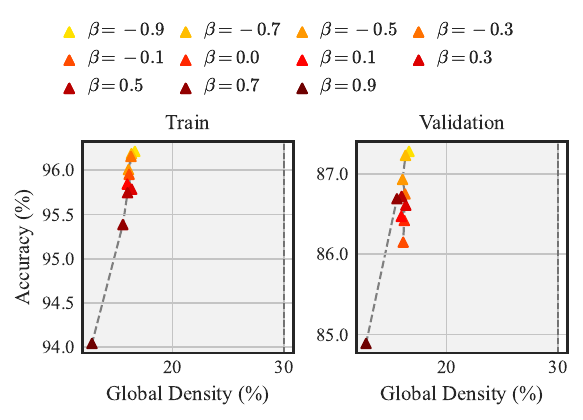}
    \includegraphics[scale=1,trim={0 0 0 18mm},clip]{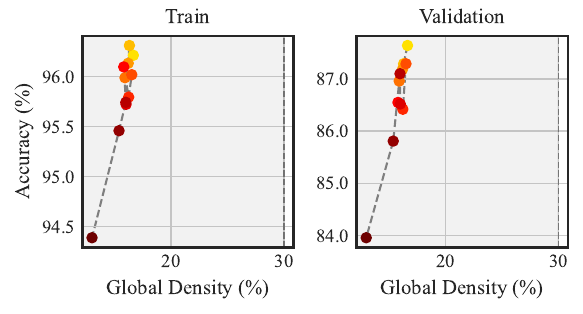}
    \caption{Trade-off plot under a 30\% global density target for \Nesterov~(top) and \Polyak~(bottom) momentum.}
    \label{fig:icml2024:momentum_ablation_global_30}
\end{figure}

\begin{table*}[h!]
\centering
\caption{Ablation on the momentum hyperparameter for \Nesterov~on a CIFAR10 task with a 30\% global density target.}
\label{tab:icml2024:nesterov_ablation_global_30}
\begin{adjustbox}{max width=1\textwidth}
\begin{tabular}{c|cc|cc}
\toprule
Nesterov $\beta$ & Train Acc. & Test Acc. & Violation \\
\midrule
-0.9 & 96.21 & 87.28 & -13.36 \\
-0.7 & 96.19 & 87.23 & -13.67 \\
-0.5 & 96.01 & 86.93 & -13.93 \\
-0.3 & 96.16 & 86.75 & -13.71 \\
-0.1 & 95.95 & 86.15 & -13.88 \\
0.0 & 95.79 & 86.42 & -13.78 \\
0.1 & 95.84 & 86.47 & -14.05 \\
0.3 & 95.79 & 86.61 & -13.64 \\
0.5 & 95.75 & 86.72 & -14.02 \\
0.7 & 95.39 & 86.69 & -14.44 \\
0.9 & 94.04 & 84.89 & -17.20 \\
\bottomrule
\end{tabular}
\end{adjustbox}
\end{table*}

\begin{table*}[h!]
\centering
\caption{Ablation on the momentum hyperparameter for \Polyak~on a CIFAR10 task with a 30\% global density target.}
\label{tab:icml2024:polyak_ablation_global_30}
\begin{adjustbox}{max width=1\textwidth}
\begin{tabular}{c|cc|cc}
\toprule
Polyak $\beta$  & Train Acc. & Test Acc. & Violation \\
\midrule
-0.9 & 96.21 & 87.64 & -13.36 \\
-0.7 & 96.31 & 87.29 & -13.71 \\
-0.5 & 96.13 & 87.18 & -13.82 \\
-0.3 & 95.99 & 86.96 & -14.10 \\
-0.1 & 96.02 & 87.29 & -13.50 \\
0.0 & 95.79 & 86.42 & -13.78 \\
0.1 & 96.10 & 86.55 & -14.21 \\
0.3 & 95.72 & 86.52 & -14.00 \\
0.5 & 95.74 & 87.10 & -14.03 \\
0.7 & 95.46 & 85.81 & -14.63 \\
0.9 & 94.39 & 83.96 & -17.02 \\
\bottomrule
\end{tabular}
\end{adjustbox}
\end{table*}

\end{document}